\documentclass[sigconf]{acmart}

\AtBeginDocument{%
  }

\setcopyright{acmcopyright}
\copyrightyear{2023}
\acmYear{2023}
\setcopyright{acmlicensed}
\acmConference[KDD '23]{Proceedings of the 29th ACM SIGKDD Conference on Knowledge Discovery and Data Mining}{August 6--10, 2023}{Long Beach, CA, USA}
\acmBooktitle{Proceedings of the 29th ACM SIGKDD Conference on Knowledge Discovery and Data Mining (KDD '23), August 6--10, 2023, Long Beach, CA, USA}
\acmPrice{15.00}
\acmDOI{10.1145/3580305.3599488}
\acmISBN{979-8-4007-0103-0/23/08}

\acmSubmissionID{rtfp0337}

\newcommand\TODO[1][]{{\color{orange}[TODO\ifthenelse{\equal{#1}{}}{}{: #1}]}}

\usepackage{ifthen}
\usepackage{multirow}
\usepackage{algorithm}
\usepackage{algorithmic}
\usepackage{subcaption}
\usepackage{balance}
\newcommand\BE\textbf
\newcommand\BM\boldsymbol
\newcommand\BB\mathbb
\newcommand\CAL\mathcal
\newcommand\RM\mathrm
\newcommand\HAT\widehat
\newcommand\BAR\overline
\newcommand\TLD\widetilde
\newcommand\Poly{\operatorname{poly}}
\newcommand\Tp{^{\mathsf T}}

\newcommand\OP[1]{\mathop{\operatorname{#1}}}
\newcommand\EQ[1]{\begin{equation}#1\end{equation}}

\newcommand\AL[1]{\begin{align}#1\end{align}}

\DeclareMathOperator*{\argmax}{argmax}

\DeclareMathOperator*{\Exp}{\mathbb{E}}
\DeclareMathOperator*{\Var}{\mathsf{Var}}
\newcommand\DeferProof{}
\newcommand\SIRS{\texttt{S}}
\newcommand\SIRI{\texttt{I}}
\newcommand\SIRR{\texttt{R}}
\newtheorem{THM}{\textbf{Theorem}}
\newtheorem{LEM}[THM]{\textbf{Lemma}}
\newtheorem{PRP}[THM]{\textbf{Proposition}}

\newtheorem{PRB}{\textbf{Problem}}
\newcommand\THMref[1]{\textsc{Theorem}~\ref{THM:#1}}
\newcommand\LEMref[1]{\textsc{Lemma}~\ref{LEM:#1}}
\newcommand\PRPref[1]{\textsc{Proposition}~\ref{PRP:#1}}

\newcommand\PRBref[1]{\textsc{Problem}~\ref{PRB:#1}}
\newcommand\RQref[1]{RQ\ref{RQ:#1}}

\newcommand\ALG[3]{\begin{algorithm}[t]\caption{#2}\label{alg:#1}\begin{algorithmic}[1]#3\end{algorithmic}\end{algorithm}}
\newcommand\ALGref[1]{Algorithm~\ref{alg:#1}}

\newcommand\problem{\textsc{DASH}}
\newcommand\method{\textsc{DITTO}} 

\begin{document}

\title{Reconstructing Graph Diffusion History from a Single Snapshot}

\author{Ruizhong Qiu}
\email{rq5@illinois.edu}
\affiliation{\institution{University of Illinois}\city{Urbana-Champaign}\state{IL}\country{USA}}

\author{Dingsu Wang}
\email{dingsuw2@illinois.edu}
\affiliation{\institution{University of Illinois}\city{Urbana-Champaign}\state{IL}\country{USA}}

\author{Lei Ying}
\email{leiying@umich.edu}
\affiliation{\institution{University of Michigan}\city{Ann Arbor}\state{MI}\country{USA}}

\author{H. Vincent Poor}
\email{poor@princeton.edu}
\affiliation{\institution{Princeton University}\city{Princeton}\state{NJ}\country{USA}}

\author{Yifang Zhang}
\email{zhang303@illinois.edu}
\affiliation{\institution{C3.ai Digital Transformation Institute}\city{Urbana}\state{IL}\country{USA}}

\author{Hanghang Tong}
\email{htong@illinois.edu}
\affiliation{\institution{University of Illinois}\city{Urbana-Champaign}\state{IL}\country{USA}}


\begin{abstract}

Diffusion on graphs is ubiquitous with numerous high-impact applications, ranging from the study of residential segregation in socioeconomics and activation cascading in neuroscience, to the modeling of disease contagion in epidemiology and malware spreading in cybersecurity. In these applications, complete diffusion histories play an essential role in terms of identifying dynamical patterns, reflecting on precaution actions, and forecasting intervention effects.
Despite their importance, complete diffusion histories are rarely available and are highly challenging to reconstruct due to ill-posedness, explosive search space, and scarcity of training data. To date, few methods exist for diffusion history reconstruction. They are exclusively based on the maximum likelihood estimation (MLE) formulation and require to know true diffusion parameters. 
In this paper, we study an even harder problem, namely \emph{reconstructing \underline Diffusion history from \underline A single \underline Snaps\underline Hot} (\problem), where we seek to reconstruct the history from only the final snapshot without knowing true diffusion parameters. We start with theoretical analyses that reveal a fundamental limitation of the MLE formulation. We prove: (a) estimation error of diffusion parameters is unavoidable due to NP-hardness of diffusion parameter estimation, and (b) the MLE formulation is sensitive to estimation error of diffusion parameters. 
To overcome the inherent limitation of the MLE formulation, we propose a novel \emph{barycenter formulation}: finding the barycenter of the posterior distribution of histories, which is provably stable against the estimation error of diffusion parameters.
We further develop an effective solver named \emph{\underline{DI}ffusion hi\underline Tting \underline Times with \underline Optimal proposal} (\method{}) by reducing the problem to estimating posterior expected hitting times via the Metropolis--Hastings Markov chain Monte Carlo method (M--H MCMC) and employing an unsupervised graph neural network to learn an optimal proposal to accelerate the convergence of M--H MCMC. 
We conduct extensive experiments to demonstrate the efficacy of the proposed method.
Our code is available at \url{https://github.com/q-rz/KDD23-DITTO}. 

\end{abstract}



\maketitle

\section{Introduction}

\begin{figure}\centering
\includegraphics[width=\columnwidth]{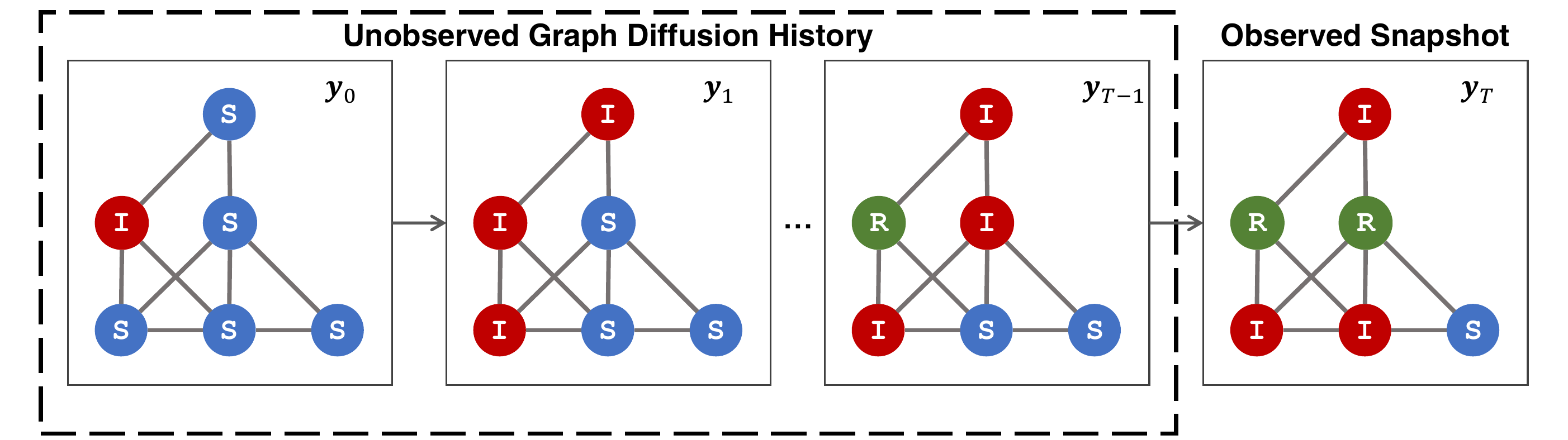}
\vspace{-15pt}
\caption{Illustration of the \problem{} problem. This is an SIR diffusion process on a graph, where each square box represents a snapshot $\BM y_t$ at each time $t$. In the \problem{} problem, only the final snapshot $\BM y_T$ is observed, and we need to reconstruct all the unobserved snapshots $\BM y_0,\BM y_1,\dots,\BM y_{T-1}$.}
\label{fig:dash-illus}
\end{figure}

Diffusion on graphs is ubiquitous in various domains owing to its generality in representing complex dynamics among interconnected objects. It appears in numerous high-impact applications, ranging from the study of residential segregation in socioeconomics \cite{schelling1978micromotives} and activation cascading in neuroscience \cite{avena2018communication}, to the modeling of disease contagion in epidemiology \cite{klovdahl1985social} and malware spreading in cybersecurity \cite{wang2009understanding}. At the core of these applications are the {\em complete diffusion histories} of the underlying diffusion process, which could be exploited to identify dynamical patterns \cite{coleman1957diffusion}, reflect on precaution actions \cite{chen2015node}, forecast intervention effects \cite{valente2005network}, etc.

Despite their 
importance, complete diffusion histories are rarely available in real-world applications because diffusion may not be noticed in early stages, collecting diffusion histories 
may incur unaffordable costs, and/or tracing node states may raise privacy concerns \cite{pcvc,ssr}. 
It is thus highly desirable to develop learning-based 
methods to automatically reconstruct diffusion histories from limited observations. 
However, diffusion history reconstruction faces 
critical challenges. 
\textbf{(i) Ill-posed inverse problem.}
Since different histories can result in the same observation, it is difficult to distinguish which history is preferred. Hence, it is crucial to design 
a formulation with desired inductive bias. \textbf{(ii) Explosive search space.} The number of possible histories grows exponentially with the number of nodes, so history estimation is an extremely high-dimensional combinatorial 
problem. \textbf{(iii) Scarcity of training data.} Conventional methods for time series imputation such as supervised learning (e.g., \cite{brits,net3,grin,spin}) 
require training data to learn from, but true diffusion histories are rarely available. Thus, they turn 
out to be inefficacious or even inapplicable for this problem.

Compared to extensive research on forward problems on diffusion (e.g., node immunization \cite{hayashi2004oscillatory}; see Sec.~\ref{sec:related-diffus} for a survey), 
few works exist for diffusion history reconstruction. The sparse literature on diffusion history reconstruction is 
exclusively based on the maximum likelihood estimation (MLE) formulation \cite{pcvc,ssr}, which relies on two assumptions: (1) knowing true diffusion parameters and/or (2) knowing partial diffusion histories. 
However, neither diffusion parameters or partial diffusion history is available in many real-world applications. For example, when an epidemic is noticed for the first time, we might only know who are \emph{currently} infected but have no historical infection information (e.g., who the patient zero is, when outbreaks happen,
and where super-spreaders locate). 

To address these limitations, 
we study a novel and more realistic setting: 
reconstructing a complete \underline Diffusion history from \underline A single \underline Snaps\underline Hot (\problem). 
Importantly, we remove both assumptions of 
existing works. 
That is, we do not require knowing true diffusion 
parameters, and we only have access to the current snapshot. 

Under such a realistic setting, we start with theoretical analyses revealing a fundamental limitation of the MLE formulation for the \problem{} problem, which motivates us to propose a novel \emph{barycenter formulation} that addresses the limitation. We then develop an effective solver named \emph{\underline{DI}ffusion hi\underline Tting \underline Times with \underline Optimal proposal} (\method) for the barycenter formulation. Our method is unsupervised and thus does not require real diffusion histories as training data, which is desirable due to the scarcity of training data. Extensive experiments demonstrate the efficacy and the scalability of our method \method. The main contributions of this paper are:
\begin{itemize}
\item\textbf{Problem definition.} Motivated by real-world 
challenges, we propose and study a new 
problem \problem. This challenging problem assumes that only the current snapshot is observed, while previous works rely on partial diffusion histories and/or require to know true diffusion parameters. 
\item\textbf{Theoretical 
insights.} We reveal a fundamental limitation of the MLE formulation for \problem{}. We 
prove: (a) estimation error of diffusion parameters is inevitable due to NP-hardness of diffusion parameter estimation, and (b) the MLE formulation is sensitive to estimation error of diffusion parameters. 
\item\textbf{Methodology.} 
We propose a novel \emph{barycenter formulation} for \problem{}, which is \emph{provably stable} against the estimation error of diffusion parameters. We further develop an effective method 
\method{}, which reduces \problem{} to estimating hitting times via MCMC and employs an unsupervised GNN to learn an optimal proposal to accelerate the convergence of MCMC. \method{} has time complexity $\OP O(T(n\log n+m))$ scaling near-linearly w.r.t.\ the output size $\Theta(Tn)$, where $T$ is the timespan, and $n,m$ are the numbers of nodes and edges, respectively. 
\item\textbf{
Evaluation.} We conduct extensive experiments with both synthetic and real-world datasets, 
and \method{} consistently achieves strong performance on all datasets. For example, \method{} is 10.06\% better than the best baseline on the Covid dataset in terms of normalized rooted mean squared error.
\end{itemize}

\section{Problem Definition}

\begin{table}\centering\footnotesize
\caption{Main notations.}\label{tab:notation}
\vspace{-8pt}
\begin{tabular}{ll}
\toprule
\textbf{Symbol}&\textbf{Definition}\\
\midrule
$\CAL T$&the set of time\\
$\CAL X$&the set of diffusion states\\
$\CAL V$&the set of nodes\\
$\CAL E$&the set of edges\\
$\CAL N_u$&the set of neighbors of node $u$\\
\midrule
$T$&the timespan of interest\\
$n^{\CAL X_0}$&the number of nodes with state in $\CAL X_0$\\
$y_{t,u}$&the state of node $u$ at time $t$\\
$\BM y_t$&a snapshot of diffusion at time $t$\\
$\BM Y$&a complete diffusion history\\
$\HAT{\BM Y}$&the reconstructed diffusion history\\
\midrule
$\SIRS,\,\SIRI,\,\SIRR$&states in the SIR model\\
$\beta^\SIRI,\,\beta^\SIRR$&the infection rate and the recovery rate\\
$\BM\beta$&true diffusion parameters\\
$\HAT{\BM\beta}$&estimated diffusion parameters\\
$P_{\BM\beta}$&the probability measure of the diffusion model\\
$P_{\BM\beta}\mathbin|\BM y_T$&the posterior given the observed snapshot\\
$\OP{supp}(P)$&the set of possible histories\\
$\OP{supp}(P\mathbin|\BM y_T)$&the set of histories consistent with the snapshot $\BM y_T$\\
\midrule
$\searrow$&the one-sided limit from above\\
$\partial$&the partial derivative operator\\
$\nabla$&the gradient operator\\
$\Exp$&the expectation operator\\
$\OP O,\,\Theta$&the asymptotic notations\\
\bottomrule
\end{tabular}
\end{table}

In this section, we formally define the \problem{} problem. 
Our notation conventions are as follows. We use the calligraphic font for sets (e.g., $\mathcal{E}$); lightface uppercase letters for constants (e.g., $T$) and probability distributions (e.g., $P$); lightface lowercase letters for indices (e.g., $t$) and scalar-valued functions (e.g., $\psi$); boldface lowercase letters for vectors (e.g., $\BM\beta$); boldface uppercase letters for matrices (e.g., $\BM Y$); the monospaced font for states (e.g., \SIRS); and the hat notation for estimates (e.g., $\HAT{\BM\beta}$). Notations are summarized in Table~\ref{tab:notation}.

\subsection{Preliminaries}

\subsubsection{Diffusion on Graphs}
We study discrete-time diffusion processes, where $\CAL T:=\{0,1,\dots,T\}$ denotes the set of time, and $\CAL X$ denotes the set of diffusion states. The graph is \emph{undirected}, with node set $\CAL V$ and edge set $\CAL E$, where 
the number of nodes is $|\CAL V|=n$, and the number of edges is $|\CAL E|=m$. For each node $u\in\CAL V$, let $\CAL N_u:=\{v\in\CAL V:(u,v)\in\CAL E\}$ denote the neighbors of node $u$.

Let $y_{t,u}\in\CAL X$ denote the state of a node $u\in\CAL V$ at a time $t\in\CAL T$. A \emph{diffusion process} \cite{wang2003epidemic} 
on a graph $(\CAL V,\CAL E)$ is a spatiotemporal stochastic process $\langle y_{t,u}\rangle_{t\in\CAL T,\,u\in\CAL V}$ where $y_{t,u}$ depends only on $\big\{y_{t-1,v}:v\in\{u\}\cup\CAL N_u\big\}$ for every node $u\in\CAL V$ at every time $t>0$. Hence, a diffusion process is necessarily a Markov process. A \emph{snapshot} at a time $t\in\CAL T$ is a vector $\BM y_t:=(y_{t,u})_{u\in\CAL V}\in\CAL X^{\CAL V}$ 
containing all node states at the time $t$. A \emph{diffusion history} (or a \emph{history} for short) is a matrix $\BM Y:=(\BM y_0,\dots,\BM y_T)\Tp=(y_{t,u})_{t\in\CAL T,\,u\in\CAL V}\in\CAL X^{\CAL T\times\CAL V}$ containing snapshots at all times. A history $\BM Y$ is said to be \emph{feasible} iff it happens with nonzero probability.

\subsubsection{Graph Diffusion Models}
In this work, we focus on two classic graph diffusion models, namely the Susceptible--Infected (SI) model and the Susceptible--Infected--Recovered (SIR) model \cite{kermack1927contribution}. The SI model can be considered as a special case of the SIR model, so we start with the SIR model.

The SIR model describes the contagion process of infectious diseases from which recovery provides permanent protection against re-infection. In the SIR model, the states are $\CAL X:=\{\SIRS,\SIRI,\SIRR\}$. The probability measure $P_{\BM\beta}$ for the SIR model 
is parameterized by two parameters $\BM\beta:=(\beta^\SIRI,\beta^\SIRR)\Tp\in(0,1)^2$, where $\beta^\SIRI$ and $\beta^\SIRR$ are called the \emph{infection rate} and the \emph{recovery rate}, respectively. Let $n^{\CAL X_0}(\BM y_t)$ denote the number of nodes with state in $\CAL X_0$ at time $t$, e.g., $n^{\SIRI\SIRR}(\BM y_t)$ meaning the number of infected or recovered nodes at time $t$. Please refer to Appendix~\ref{app:sir} for the definition of $P_{\BM\beta}$ in the SIR model.

Let $\OP{supp}(P):=\{\BM Y\in\CAL X^{\CAL T\times\CAL V}:P_{\BM\beta}[\BM Y]>0\}$ denote the set of possible histories. For a snapshot $\BM y_T\in\CAL X^\CAL V\!$, let $P_{\BM\beta}\mathbin|\BM y_T:=P_{\BM\beta}[\,\cdot\mid\BM y_T]$ denote the posterior given the snapshot $\BM y_T$, and let $\OP{supp}(P\mathbin|\BM y_T):=\{\BM Y\in\CAL X^{\CAL T\times\CAL V}:P_{\BM\beta}[\BM Y\mid\BM y_T]>0\}$ denote the set of possible histories consistent with the snapshot $\BM y_T$. Since the set of possible histories does not depend on $\BM\beta$, we omit $\BM\beta$ in its notation.

For the SI model, it is defined by letting $\beta^\SIRR:=0$ and removing the state $\SIRR$ from the SIR model. It describes the contagion process of infectious diseases that cannot recover.

\subsection{Problem Statement}
The problem we study is reconstructing the complete diffusion history $\BM Y$ from a single snapshot $\BM y_T$ without knowing true diffusion parameters $\BM\beta$. As is discussed above, it is often impractical to obtain true diffusion histories for supervised methods to learn from or for statistical methods to accurately estimate diffusion parameters from. Hence, we assume that neither a database of diffusion histories nor true diffusion parameters are known. Instead, what we know is the final snapshot $\BM y_T$. Since a single snapshot cannot provide any information about the underlying diffusion model, then we have to assume that the underlying diffusion model is known as domain knowledge while its diffusion parameters are not assumed to be known.
This is a common assumption in previous work \cite{pcvc,ssr}. 
Such domain knowledge is usually available in practice. For instance, if we know that recovery from a disease will probably provide lifelong protection (e.g., chickenpox), then it will be reasonable to assume that the underlying diffusion model is the SIR model while we do not know its diffusion parameters. In this work, we consider the SI model and the SIR model.

Given a timespan $T$ of interest and the snapshot $\BM y_T$ at time $T$, our task is to reconstruct the diffusion history $\BM y_0,\dots,\BM y_{T-1}$. Since we assume no extra 
diffusion histories for training, our method is of unsupervised learning. This setting is more realistic than that of previous works (e.g., \cite{pcvc,ssr}), 
where diffusion histories for training and/or true diffusion parameters are assumed to be available.

Under this realistic setting, the diffusion history reconstruction problem  involves two aspects: (i) estimating diffusion parameters from a single snapshot and (ii) reconstructing the diffusion history in the presence of estimation error of diffusion parameters. As we will show later in \THMref{diffus-par-mle-nphard}, it is NP-hard to 
estimate diffusion parameters from a snapshot. Thus, diffusion parameter estimation itself is a non-trivial problem here. Furthermore, its NP-hardness implies that estimation error of diffusion parameters is unavoidable. Hence, it is important for the diffusion history reconstruction method to be stable against estimation error of diffusion parameters.

We assume that the source nodes are unknown, but the initial distribution $P[\BM y_0]$ is known as a domain knowledge, such as how many nodes (roughly) are initially infected, which areas the source nodes probably locate in, and whether high-density communities are more suitable for epidemics to break out. 
The knowledge of $P[\BM y_0]$ is necessary because without it the diffusion parameters will be uncertain in the unsupervised setting. For instance, given the timespan $T$ and the snapshot $\BM y_T$, a smaller number of initially infected nodes suggests a higher infection rate, while a larger number of initially infected nodes implies a lower infection rate. 

For computational consideration, the initial distribution should be efficiently computable up to a normalizing constant, i.e., $P[\BM y_0]\propto p(\BM y_0)$ for all possible $\BM y_0$ where $p:\CAL X^\CAL V\to\BB R_{\ge0}$ is an efficiently computable function.
In this work, we define the initial distribution w.r.t.\ 
the number $n_0^\SIRI$ of initially infected nodes. We assume 
no initially recovered nodes, because they could be removed from the graph. 
Thus, we define $P_{\BM\beta}[\BM y_0]\propto\exp({-\gamma|n^\SIRI(\BM y_0)-n_0^\SIRI|-\gamma n^\SIRR(\BM y_0)})$, where $\gamma>0$ is a hyperparameter. 
If we are more certain on $n_0^\SIRI$, 
we should use a larger $\gamma$. We do not consider $n_0^\SIRI$ 
a hard constraint because 
it is typically a rough estimate rather than the exact number.




We formally state our problem definition as \PRBref{dash} below. See Fig.~\ref{fig:dash-illus} for an illustration of the \problem{} problem.

\vspace{-.25em}
\begin{PRB}[\problem]\label{PRB:dash}
Under SI/SIR model, reconstruct the complete \underline Diffusion history from \underline A single \underline Snaps\underline Hot without knowing true diffusion parameters. \textbf{Input:} (i) graph $(\CAL V,\CAL E)$; (ii) timespan $T$ of interest; (iii) final snapshot $\BM y_T\in\CAL X^\CAL V$; (iv) initial distribution $P[\BM y_0]$. \textbf{Output:} reconstructed complete diffusion history $\HAT{\BM Y}\in\CAL X^{\CAL T\times\CAL V}$.
\end{PRB}
\vspace{-.5em}

\section{Revisiting Diffusion History MLE}
In this section, we theoretically reveal a fundamental limitation of the MLE formulation for diffusion history reconstruction. In Section~\ref{ssec:mle-np}, we show that estimation error of diffusion parameters is unavoidable due to the NP-hardness of diffusion parameter estimation. Then in Section~\ref{ssec:mle-sens}, we prove that the MLE formulation for diffusion history reconstruction is sensitive to estimation error of diffusion parameters. 
Therefore, the performance of the MLE formulation can be drastically degraded by such estimation error of diffusion parameters. Please refer to Appendix~\ref{app:pf} for proofs. 

\subsection{NP-Hardness of Diffusion Parameter Estimation}\label{ssec:mle-np}
In this subsection, we show that estimation error of diffusion parameters is unavoidable due to the NP-hardness of diffusion parameter estimation. To estimate diffusion parameters $\BM\beta$, a conventional approach is maximum likelihood estimation (MLE) \cite{ratcliff2002estimating}. 
Given the observed snapshot $\BM y_T$, diffusion parameter MLE is formulated as:
\EQ{\vspace{-1mm}
\max_{\HAT{\BM\beta}}P_{\HAT{\BM\beta}}[\BM y_T].
\label{eq:diffus-par-mle-form}\vspace{-1mm}}
To optimize Eq.~\eqref{eq:diffus-par-mle-form}, one may consider using gradient-based methods. Typically, gradient-based methods first evaluate the likelihood function and then differentiate 
it to get the gradient. However, due to the explosive search space of possible histories, it is intractable to compute the likelihood $P_{\HAT{\BM\beta}}[\BM y_T]$. In fact, we prove that computing the likelihood $P_{\HAT{\BM\beta}}[\BM y_T]$ (or even approximating it) 
is already NP-hard, as is stated in \THMref{diffus-prob-nphard}. 

\begin{THM}[NP-hardness of snapshot probability]\label{THM:diffus-prob-nphard}
Under the SIR model, approximating the probability of a snapshot\footnote{See \PRBref{diffus-prob} in Appendix~\ref{app:pf-1} for the precise definition.} is NP-hard, even if the initial probability $P[\BM y_0]$ (including its normalizing constant) for each possible $\BM y_0$ can be computed 
in polynomial time.
\end{THM}
\DeferProof

\THMref{diffus-prob-nphard} implies that there probably do not exist tractable algorithms to approximate the likelihood $P_{\HAT{\BM\beta}}[\BM y_T]$, unless $\mathrm{P=NP}$. The intuition behind \THMref{diffus-prob-nphard} is that possible diffusion histories form an explosively large search space. Since gradient-based methods require computing the likelihood $P_{\HAT{\BM\beta}}[\BM y_T]$, this diminishes 
the applicability of
such methods for diffusion parameter MLE.

Although approximating the likelihood $P_{\HAT{\BM\beta}}[\BM Y]$ is intractable, one may also wonder whether there exists an efficient algorithm that can give the optimal $\HAT{\BM\beta}$ \emph{without} computing $P_{\HAT{\BM\beta}}[\BM Y]$. Unfortunately, we prove that computing MLE diffusion parameters (even if 
a small relative error is allowed) is also NP-hard, as is stated in \THMref{diffus-par-mle-nphard}.

\begin{THM}[NP-hardness of diffusion parameter MLE]\label{THM:diffus-par-mle-nphard}
Under the SIR model, diffusion parameter MLE\footnote{See \PRBref{diffus-par-mle} in Appendix~\ref{app:pf-2} for the precise definition.} is NP-hard, even if the initial probability $P[\BM y_0]$ (up to a normalizing constant\footnote{The normalizing constant does not affect the result of this problem.}) for each possible $\BM y_0$ can be computed in polynomial time.
\end{THM}
\DeferProof

Both \THMref{diffus-prob-nphard} and \THMref{diffus-par-mle-nphard} suggest that there do not exist tractable algorithms to estimate diffusion parameters $\BM\beta$ accurately from a single snapshot $\BM y_T$, unless $\RM{P=NP}$. Hence, estimation error of diffusion parameters is unavoidable in the \problem{} problem. Consequently, a good method for the \problem{} problem should be stable against such estimation error of diffusion parameters, which motivates us to utilize posterior expected hitting times in Section~\ref{ssec:bary-form}.

\subsection{Sensitivity to Estimation Error of Diffusion Parameters}\label{ssec:mle-sens}

In this subsection, we reveal a fundamental limitation of the MLE formulation for \problem{}.
The MLE formulation reconstructs the history $\BM Y$ by maximizing its likelihood $P_{\HAT{\BM\beta}}[\BM Y]$ among all possible histories that are consistent with the observed history $\BM y_T$: 
\AL{
\max_{\BM Y\in\OP{supp}(P\mid\BM y_T)}P_{\HAT{\BM\beta}}[\BM Y].
}
As is shown in Section~\ref{ssec:mle-np}, estimation error of diffusion parameters is unavoidable. Thus, it is crucial to analyze the sensitivity of the MLE formulation to such estimation error of diffusion parameters. We prove that unfortunately, the MLE formulation is sensitive to estimation error of diffusion parameters when diffusion parameters are small, as is stated 
in \THMref{diffus-mle-sens}.

\begin{THM}[Sensitivity to estimation error of diffusion parameters]\label{THM:diffus-mle-sens}
Under the SIR model with small true $\BM\beta$ (i.e., $\BM\beta\searrow\BM0$), for every possible history $\BM Y$, we have:
\AL{
\frac{\partial}{\partial\beta^\SIRI}P_{\BM\beta}[\BM Y]&=\Theta\Big(\frac1{\beta^\SIRI}\Big)P_{\BM\beta}[\BM Y]\quad\textnormal{if }n^{\SIRI\SIRR}(\BM y_T)>n^{\SIRI\SIRR}(\BM y_0);\\
\frac{\partial}{\partial\beta^\SIRR}P_{\BM\beta}[\BM Y]&=\Theta\Big(\frac1{\beta^\SIRR}\Big)P_{\BM\beta}[\BM Y]\quad\textnormal{if }n^{\SIRR}(\BM y_T)>n^{\SIRR}(\BM y_0).
}
\end{THM}
\DeferProof

\THMref{diffus-mle-sens} shows that the relative error of the likelihood induced by estimation error of diffusion parameters is inversely proportional to true diffusion parameters. The conditions $n^{\SIRI\SIRR}(\BM y_T)>n^{\SIRI\SIRR}(\BM y_0)$ and $n^{\SIRR}(\BM y_T)>n^{\SIRR}(\BM y_0)$ mean that infection and recovery do happen during the timespan $T$ of interest, which is almost always the case in practice. Hence, the conditions are quite mild and realistic. 

Infection and recovery rates in many real-world data are 
small \cite{o1987epidemiology,goh2009inflammatory,vanderpump2011epidemiology}. Hence, the likelihood under estimated diffusion parameters $\HAT{\BM\beta}$ has a large relative error and is ill-conditioned. Moreover, since the error of the likelihood is proportional to the likelihood itself, the MLE history under \emph{estimated} $\HAT{\BM\beta}$ 
may have larger decrease in likelihood than other histories and thus may not be the MLE history under \emph{true} diffusion parameters. 
Therefore, 
sensitivity 
is indeed a fundamental limitation of 
the MLE formulation 
in practice.

To address this limitation, we instead 
solve the \problem{} problem from a new perspective and  propose a so-called \emph{barycenter formulation} utilizing posterior expected hitting times, 
which, as we will prove, is stable against estimation error of diffusion parameters.

\section{Proposed Method: \method}
In this section, we propose a method called \emph{\underline{DI}ffusion hi\underline Tting \underline Times with \underline Optimal proposal} (\method) for solving the \problem{} problem. In Sec.~\ref{sec:method-beta}, we employ \emph{mean-field approximation} to estimate diffusion parameters. In Sec.~\ref{ssec:bary-form}, we propose the \emph{barycenter formulation} that is provably stable, and reduce the DASH problem to estimating the posterior expected \emph{hitting times}. In Sec.~\ref{sec:method-mcmc}, we propose to use an unsupervised graph neural network to learn an optimal proposal in Metropolis--Hastings Markov chain Monte Carlo (M--H MCMC) algorithm to estimate the posterior expected \emph{hitting times}. The overall procedure of \method{} is presented in \ALGref{method}. 


\ALG{method}{Proposed method: \method}{
\REQUIRE{(i) the graph $(\CAL V,\CAL E)$; (ii) the timespan $T$ of interest and the observed snapshot $\BM y_T$; (iii) the initial distribution $P[\BM y_0]$ and the (rough) number $n^\SIRI_0$ of initial infections; (iv) the batch sizes $K,L$, the MCMC steps $S$, and the moving average hyperparameter $\eta$.}
\ENSURE{the reconstructed diffusion history $\HAT{\BM Y}$.}
\STATE initialize the diffusion parameter estimates $\HAT{\BM\beta}$
\WHILE{$\HAT{\BM\beta}$ not converged}
    \STATE initialize pseudolikelihoods $f^{\SIRS}_{0,u;\HAT{\BM\beta}},\,f^{\SIRI}_{0,u;\HAT{\BM\beta}},\,f^{\SIRR}_{0,u;\HAT{\BM\beta}}$ for all $u\in\CAL V$ by Eq.~\eqref{eq:mf-init}
    \FOR{$t=0,\dots,T-1$}
        \STATE compute pseudolikelihoods $f^{\SIRS}_{t+1,u;\HAT{\BM\beta}},\,f^{\SIRI}_{t+1,u;\HAT{\BM\beta}},\,f^{\SIRR}_{t+1,u;\HAT{\BM\beta}}$ for all $u\in\CAL V$ by Eq.'s~\eqref{eq:mf-s}\eqref{eq:mf-i}\eqref{eq:mf-r}
    \ENDFOR
    \STATE update $\HAT{\BM\beta}\gets\OP{GradientDescent}\!\Big({-\frac1n\sum\limits_{u\in\CAL V}\!\!\log f^{y_{T,u}}_{T,u;\HAT{\BM\beta}}}\Big)$
\ENDWHILE
\STATE initialize the proposal $Q_{\BM\theta}$
\WHILE{$Q_{\BM\theta}$ not converged}
    \STATE sample $K$ histories $\BM Y^{(1)},\dots,\BM Y^{(K)}\sim P_{\HAT{\BM\beta}}$
    \STATE update $\BM\theta\gets\OP{GradientDescent}\!\Big({-\frac1K\sum\limits_{i=1}^K\log Q_{\BM\theta}(\BM y_T^{(i)})[\BM Y^{(i)}]}\Big)$
\ENDWHILE
\STATE sample $L$ histories $\BM Y^{(0,1)},\dots,\BM Y^{(0,L)}\sim Q_{\BM\theta}(\BM y_T)$
\STATE initialize the hitting time estimates: for each  $u\in\CAL V$\\
$\HAT h_u^\SIRI\gets\frac1L\sum\limits_{i=1}^Lh_u^\SIRI(\BM Y^{(0,i)})$,\quad$\HAT h_u^\SIRR\gets\frac1L\sum\limits_{i=1}^Lh_u^\SIRR(\BM Y^{(0,i)})$\quad
\FOR{$s=1,\dots,S$}
    \STATE sample $L$ histories $\TLD{\BM Y}{}^{(s,1)},\dots,\TLD{\BM Y}{}^{(s,L)}\sim Q_{\BM\theta}(\BM y_T)$
    \STATE generate $\xi^{(s,1)},\dots,\xi^{(s,L)}\sim\OP{Uniform}[0,1)$
    \STATE update MCMC by the M--H rule: for each $i=1,\dots,L$\\$\BM Y^{(s,i)}\!\!\gets\!\begin{cases}\TLD{\BM Y}{}^{(s,i)}&\!\!\text{if }\xi^{(s,i)}\!\!<\!\frac{P_{\HAT{\BM\beta}}[\TLD{\BM Y}{}^{(s,i)}]Q_{\BM\theta}(\BM y_T)[\BM Y^{(s-1,i)}]}{P_{\HAT{\BM\beta}}[\BM Y^{(s-1,i)}]Q_{\BM\theta}(\BM y_T)[\TLD{\BM Y}{}^{(s,i)}]}\\\BM Y^{(s-1,i)}&\!\!\text{otherwise}\end{cases}$
    \STATE update the hitting time estimates: for each $u\in\CAL V$\\$\HAT h_u^\SIRI\gets\eta\HAT h_u^\SIRI\!+\!\frac{1-\eta}L\!\!\sum\limits_{i=1}^L\!h_u^\SIRI(\BM Y^{(s,i)})$,\;\,$\HAT h_u^\SIRR\gets\eta\HAT h_u^\SIRR\!+\!\frac{1-\eta}L\!\!\sum\limits_{i=1}^L\!h_u^\SIRR(\BM Y^{(s,i)})$
\ENDFOR
\STATE reconstruct the diffusion history $\HAT{\BM Y}$: for each $u\in\CAL V$\\
$\HAT y_{t,u}\gets\begin{cases}
\SIRS&\text{for }0\le t<\OP{round}(\HAT h_u^\SIRI)\\
\SIRI&\text{for }\OP{round}(\HAT h_u^\SIRI)\le t<\OP{round}(\HAT h_u^\SIRR)\\
\SIRR&\text{for }\OP{round}(\HAT h_u^\SIRR)\le t\le T
\end{cases}$
\RETURN $\HAT{\BM Y}$
}

\subsection{Mean-Field Approximation for Diffusion Parameter Estimation}\label{sec:method-beta}
Previous works \cite{ssr,pcvc} 
assume 
diffusion parameters are known, but in our setting we have to estimate the unknown diffusion parameters $\BM\beta$. 
\THMref{diffus-par-mle-nphard} shows 
it is intractable to estimate 
$\BM\beta$ via MLE. To develop a tractable estimator, we employ \emph{mean-field approximation} \cite{wang2003epidemic} to compute the so-called \emph{pseudolikelihood} 
for each node $u$ at each time $t$. 
In mean-field approximation, 
the state $y_{t,u}$ is assumed to only depend on $y_{t-1,v}$ of neighbors $v\in\CAL N_u$, but the dependence between $y_{t,u}$ and $y_{t,v}$ is ignored.
Then, 
the joint pseudolikelihood factorizes into peudolikelihoods of each single node.

Let $\HAT{\BM\beta}$ denote the estimator of diffusion parameters $\BM\beta$. Let $f_{t,u;\HAT{\BM\beta}}^x$ denote the pseudolikelihood for node $u\in\CAL V$ to be in state $x\in\CAL X$ at time $t\in\CAL T$. If we assume 
that the set of $n_0^\SIRI$ initially infected nodes is uniformly drawn from all $\binom n{n_0^\SIRI}$ possible sets, 
then the probability that a node is initially infected is $\binom{n-1}{n_0^\SIRI-1}/\binom n{n_0^\SIRI}=n_0^\SIRI/n$. Thus, 
we set the pseudolikelihoods at $t=0$ as follows:
\AL{
f^{\SIRS}_{0,u;\HAT{\BM\beta}}:=1-\frac{n_0^\SIRI}n,
\quad\;\;
f^{\SIRI}_{0,u;\HAT{\BM\beta}}:=\frac{n_0^\SIRI}n,
\quad\;\;
f^{\SIRR}_{0,u;\HAT{\BM\beta}}:=0.
\label{eq:mf-init}}
The pseudolikelihoods at time $1,\dots,T$ are computed inductively on $t$ assuming that neighbors are independent. A node is susceptible at time $t+1$ iff it is susceptible at time $t$ and is not infected by its infected neighbors at time $t+1$:
\AL{f^{\SIRS}_{t+1,u;\HAT{\BM\beta}}\!:=\!f^{\SIRS}_{t,u;\HAT{\BM\beta}}\prod_{v\in\CAL N_u}(1-f^{\SIRI}_{t,v;\HAT{\BM\beta}}\cdot\HAT\beta^\SIRI).\!\label{eq:mf-s}}
A node is infected at time $t+1$ iff it is infected at time $t$, or it is susceptible at time $t$ but is infected by one of its infected neighbors and does not recover immediately at time $t+1$:
\AL{f^{\SIRI}_{t+1,u;\HAT{\BM\beta}}:=\bigg(f^{\SIRI}_{t,u;\HAT{\BM\beta}}\!+\!f^{\SIRS}_{t,u;\HAT{\BM\beta}}\bigg(1\!-\!\!\!\prod_{v\in\CAL N_u}\!\!(1\!-\!f^{\SIRI}_{t,v;\HAT{\BM\beta}}\!\cdot\!\HAT\beta^\SIRI)\bigg)\bigg)(1\!-\!\HAT\beta^\SIRR).\label{eq:mf-i}}
A node is recovered at time $t+1$ iff it is recovered at time $t$, or it recovers just at time $t+1$:
\AL{f_{t+1,u;\HAT{\BM\beta}}^\SIRR:=f_{t,u;\HAT{\BM\beta}}^\SIRR\!+\!\bigg(f_{t,u;\HAT{\BM\beta}}^\SIRI\!+\!f_{t,u;\HAT{\BM\beta}}^\SIRS\bigg(1\!-\!\prod_{v\in\CAL N_u}\!\!(1-f_{t,v;\HAT{\BM\beta}}^\SIRI\cdot\HAT\beta^\SIRI)\bigg)\bigg)\HAT\beta^\SIRR.\label{eq:mf-r}}
Finally, we estimate $\HAT{\BM\beta}$ by maximizing the joint log-pseudolikelihood of the observed snapshot $\BM y_T$, which decomposes into the sum of log-pseudolikelihoods of each single node:
\AL{\max_{\HAT{\BM\beta}}\sum_{u\in\CAL V}\log f^{y_{T,u}}_{T,u;\HAT{\BM\beta}}.\label{eq:mf-beta}}
The objective Eq.~\eqref{eq:mf-beta} can be optimized by gradient descent methods. From now on, we let $\HAT{\BM\beta}$ denote the estimated diffusion parameters.

\subsection{Barycenter Formulation with Provable Stability}
\label{ssec:bary-form}
Since diffusion parameter estimation is NP-hard by \THMref{diffus-par-mle-nphard}, it is impossible to avoid 
estimation error of diffusion parameters. 
Meanwhile, 
as is shown by \THMref{diffus-mle-sens}, the MLE formulation for diffusion history reconstruction is sensitive to estimation error of diffusion parameters. To avoid this inherent limitation of the MLE formulation, we propose an alternative formulation that is stable against estimation error of diffusion parameters.

Since \problem{} is an ill-posed inverse problem, it is crucial to design an appropriate formulation with desired inductive bias. Here we propose a so-called \emph{barycenter formulation} that can capture the desired information of the posterior distribution of histories and is provably stable against estimation error of diffusion parameters.

Consider the \emph{hitting times} at which node states change. 
For a node $u$ in a history $\BM Y$, we define $h_u^\SIRI(\BM Y)$ and $h_u^\SIRR(\BM Y)$ to be the first time when the node $u$ becomes infected/recovered, respectively:
\AL{
h_u^\SIRI(\BM Y)&:=\min\!\big\{T+1,\,\min\{t\ge0:y_{t,u}=\SIRI\text{ or }\SIRR\}\big\},\\
h_u^\SIRR(\BM Y)&:=\min\!\big\{T+1,\,\min\{t\ge0:y_{t,u}=\SIRR\}\big\}.
}
Note that a node can become infected and recover at the same time, so the definition of $h_u^\SIRI$ includes the case where $y_{t-1,u}=\SIRS$, $y_{t,u}=\SIRR$.

Our key theoretical result is: (unlike the likelihood $P_{\HAT{\BM\beta}}[\BM Y]$,) the posterior expected hitting times are stable against estimation error of diffusion parameters, as is stated in \THMref{hit-stable}.

\begin{THM}[Stability against estimation error of diffusion parameters]\label{THM:hit-stable}
Under SIR model with small true $\BM\beta$ (i.e., $\BM\beta\!\searrow\!\!\BM0$), if $n^\SIRI\!(\BM y_0)$ and $n^\SIRR\!(\BM y_0)$ are fixed, 
then for any possible snapshot $\BM y_T$,
\AL{
\nabla_{\!\BM\beta}\!\Exp_{\BM Y\sim P_{\BM\beta}\mid\BM y_T}\!\!\!\!\!\![h_u^\SIRI(\BM Y)]=\OP O(1),\quad\nabla_{\!\BM\beta}\!\Exp_{\BM Y\sim P_{\BM\beta}\mid\BM y_T}\!\!\!\!\!\![h_u^\SIRR(\BM Y)]=\OP O(1).
}
\end{THM}
\DeferProof

In stark contrast with the MLE formulation, \THMref{hit-stable} shows 
posterior expected hitting times are stable even when $\BM\beta$ is close to zero. Such stability guarantee motivates us to utilize hitting times to design an objective function that is stable against estimation error of diffusion parameters.

Our idea is to define a distance metric $d$ for histories based on the hitting times, and find a history $\HAT{\BM Y}$ that is close to all possible histories w.r.t.\ the distance metric $d$:
\AL{\min_{\HAT{\BM Y}}\Exp_{\BM Y\sim P_{\HAT{\BM\beta}}\mid\BM y_T}[d(\HAT{\BM Y},\BM Y)^2].}
We call it the \emph{barycenter formulation}, because the optimal history $\HAT{\BM Y}$ for this formulation is the barycenter of the posterior distribution $P_{\HAT{\BM\beta}}\mathbin{|}\BM y_T$ 
w.r.t.\ the distance metric $d$.
We define the distance metric $d$ as the Euclidean distance with hitting times as coordinates:
\AL{d(\HAT{\BM Y},\BM Y):=\sqrt{\sum_{u\in\CAL V}\big((h_u^\SIRI(\HAT{\BM Y})-h_u^\SIRI(\BM Y))^2+(h_u^\SIRR(\HAT{\BM Y})-h_u^\SIRR(\BM Y))^2\big)}.}
Then, our barycenter formulation instantiates as:
\begin{align}\begin{aligned}
{}&\!\!\min_{\HAT{\BM Y}}\!\!\!\Exp_{\BM Y\sim P_{\HAT{\BM\beta}}\mid\BM y_T}\!\bigg[\!\sum_{u\in\CAL V}\!\!\big((h_u^\SIRI(\HAT{\BM Y})\!-\!h_u^\SIRI(\BM Y))^2+(h_u^\SIRR(\HAT{\BM Y})\!-\!h_u^\SIRR(\BM Y))^2\big)\bigg]\\
={}&\!\!\min_{\HAT{\BM Y}}\!\!\!\sum_{u\in\CAL V}\!\!\Big(\!\Exp_{\BM Y\sim P_{\HAT{\BM\beta}}\mid\BM y_T}\!\!\!\!\!\![(h_u^\SIRI(\HAT{\BM Y})\!-\!h_u^\SIRI(\BM Y))^2]\!+\!\!\!\!\!\!\Exp_{\BM Y\sim P_{\HAT{\BM\beta}}\mid\BM y_T}\!\!\!\!\!\![(h_u^\SIRR(\HAT{\BM Y})\!-\!h_u^\SIRR(\BM Y))^2]\Big)
.\end{aligned}
\!\!\!\!\label{eq:barycenter}\end{align}
According to bias--variance decomposition, we can further decompose Eq.~\eqref{eq:barycenter}
for $x=\SIRI,\SIRR$ as: 
\AL{\!\!\!\Exp_{\BM Y\sim P_{\HAT{\BM\beta}}\mid\BM y_T}\!\!\!\!\!\![(h_u^x(\HAT{\BM Y})-h_u^x(\BM Y))^2]&=\Big(h_u^x(\HAT{\BM Y})-\!\!\!\!\!\!\!\Exp_{\BM Y\sim P_{\HAT{\BM\beta}}\mid\BM y_T}\!\!\!\!\!\![h_u^x(\BM Y)]\Big)^{\!2}\!+\!\!\!\!\!\Var_{\BM Y\sim P_{\HAT{\BM\beta}}\mid\BM y_T}\!\!\!\!\![h_u^x(\BM Y)].}
Since the variances are constant w.r.t.\ 
the history estimator $\HAT{\BM Y}$, Eq.~\eqref{eq:barycenter} is thus equivalent to minimizing the squared biases:
\AL{
\min_{\HAT{\BM Y}}\!\!\sum_{u\in\CAL V}\!\!\Big(\Big(h_u^\SIRI(\HAT{\BM Y})-\!\!\!\!\!\Exp_{\BM Y\sim P_{\HAT{\BM\beta}}\mid\BM y_T}\!\!\!\!\!\![h_u^\SIRI(\BM Y)]\Big)^{\!2}\!\!+\!\Big(h_u^\SIRR(\HAT{\BM Y})-\!\!\!\!\!\Exp_{\BM Y\sim P_{\HAT{\BM\beta}}\mid\BM y_T}\!\!\!\!\!\![h_u^\SIRR(\BM Y)]\Big)^{\!2}\Big).}
Therefore, the optimal estimates are simply rounding each expected hitting time to the nearest integer:
\AL{
h_u^x(\HAT{\BM Y}):=\OP{round}\Big(\!\Exp_{\BM Y\sim P_{\HAT{\BM\beta}}\mid\BM y_T}\!\!\!\!\!\![h_u^x(\BM Y)]\Big),\qquad x=\SIRI,\SIRR.
\label{eq:hit-round}}
Now our problem reduces to estimating the expected hitting times over the posterior $P_{\HAT{\BM\beta}}\mathbin{|}\BM y_T$. Owing to the stability of expected hitting times, the optimal estimates $h_u^\SIRI(\HAT{\BM Y})$ and $h_u^\SIRR(\HAT{\BM Y})$ are also stable against estimation error of diffusion parameters. Finally, we reconstruct the history $\HAT{\BM Y}$ according to the estimated hitting times in Eq.~\eqref{eq:hit-round}:
\AL{
\HAT y_{t,u}:=\begin{cases}
\SIRS,&\text{for }0\le t<h_u^\SIRI(\HAT{\BM Y});\\
\SIRI,&\text{for }h_u^\SIRI(\HAT{\BM Y})\le t<h_u^\SIRR(\HAT{\BM Y});\\
\SIRR,&\text{for }h_u^\SIRR(\HAT{\BM Y})\le t\le T.
\end{cases}\label{eq:hit-to-hist}
}

\subsection{Metropolis--Hastings MCMC for Posterior Expectation Estimation}\label{sec:method-mcmc}

So far we have reduced our problem to estimating the posterior expected hitting times $\Exp_{\BM Y\sim P_{\HAT{\BM\beta}}\mid\BM y_T}[h_u^\SIRI(\BM Y)]$ and $\Exp_{\BM Y\sim P_{\HAT{\BM\beta}}\mid\BM y_T}[h_u^\SIRR(\BM Y)]$. However, due to the explosive search space of possible histories, it is intractable to compute the posterior probability $P_{\HAT{\BM\beta}}[\BM Y\mid\BM y_T]$, as is proven in \THMref{diffus-prob-nphard}. 
Therefore, it is non-trivial to design Monte Carlo samplers to estimate the posterior expectation.

To tackle this difficulty, we employ the Metropolis--Hastings Markov chain Monte Carlo (M--H MCMC) algorithm \cite{metropolis1953equation,hastings1970monte} to estimate posterior expectation. The basic idea of M--H MCMC is to construct a Markov chain whose stationary distribution is the desired posterior distribution. This algorithm requires a so-called \emph{proposal} distribution over possible histories. In our method, we design a proposal that differs for different $\BM y_T$, so we write it as $Q_{\BM\theta}(\BM y_T)[\cdot]$. Let $\OP{supp}(Q_{\BM\theta}(\BM y_T)):=\{\BM Y\in\CAL X^{\CAL T\times\CAL V}:Q_{\BM\theta}(\BM y_T)[\BM Y]>0\}$ denote the set of histories that can be generated from $Q_{\BM\theta}(\BM y_T)$. In each step of M--H MCMC, we sample a new history $\TLD{\BM Y}\sim Q_{\BM\theta}(\BM y_T)$. Let $\BM Y$ denote the current history in MCMC. Then according to the M--H rule, the new history $\TLD{\BM Y}$ is accepted with probability
\AL{
\min\!\bigg\{1,\frac{P_{\HAT{\BM\beta}}[\TLD{\BM Y}]Q_{\BM\theta}(\BM y_T)[\BM Y]}{P_{\HAT{\BM\beta}}[\BM Y]Q_{\BM\theta}(\BM y_T)[\TLD{\BM Y}]}\bigg\}.}
This defines a Markov chain of histories. After a sufficient number of steps, this Markov chain provably converges to the desired posterior distribution $P_{\HAT{\BM\beta}}\mathbin{|}\BM y_T$ \cite{hastings1970monte}. 
The convergence rate of MCMC depends critically on the design of the proposal $Q_{\BM\theta}$. If $Q_{\BM\theta}(\BM y_T)$ approximates $P_{\HAT{\BM\beta}}\mathbin{|}\BM y_T$ better, then the rate of convergence will be higher \cite{mengersen1996rates}. Since hand-craft proposals may fail to approximate the posterior distribution and thus adversely affect the convergence rate, we propose to use a graph neural network (GNN) to learn an optimal proposal. The backbone of $Q_{\BM\theta}$ is an Anisotropic GNN with edge gating mechanism \cite{bresson2018an,joshi2020learning,dimes}. The GNN takes the observed snapshot $\BM y_T$ as input and predicts a proposal $Q_{\BM\theta}(\BM y_T)$ corresponding to $\BM y_T$. The neural architecture of $Q_{\BM\theta}$ is detailed in Appendix~\ref{app:gnn-arch}. We want $Q_{\BM\theta}(\BM y_T)$ to approximate $P_{\HAT{\BM\beta}}\mathbin|\BM y_T$, so we adopt the expected squared difference of their log-likelihoods as the objective function:
\AL{
\min_{\BM\theta}\Exp_{\BM Y\sim P_{\HAT{\BM\beta}}}[(\log Q_{\BM\theta}(\BM y_T)[\BM Y]-\log P_{\HAT{\BM\beta}}[\BM Y\mid\BM y_T])^2]\label{eq:q-orig-obj}
.}
However, it is intractable to compute $P_{\HAT{\BM\beta}}[\BM Y\mid\BM y_T]$, so we cannot implement this objective function directly. To address this, 
we derive an equivalent objective (\THMref{q-equiv-obj}) that is tractable to evaluate.

\begin{THM}[An equivalent objective]\label{THM:q-equiv-obj}
If the GNN $Q_{\BM\theta}$ is sufficiently expressive 
and has the same set of possible histories as the posterior 
(i.e., $\OP{supp}(Q_{\BM\theta}(\BM y_T))=\OP{supp}(P\mathbin{|}\BM y_T)$ 
for any snapshot $\BM y_T$), then the original objective Eq.~\eqref{eq:q-orig-obj} is equivalent to
\AL{
\min_{\BM\theta}\Exp_{\BM Y\sim P_{\HAT{\BM\beta}}}\Big[\psi\Big(\frac{Q_{\BM\theta}(\BM y_T)[\BM Y]}{P_{\HAT{\BM\beta}}[\BM Y]}\Big)\Big]\label{eq:q-equiv-obj}
,}
for any \emph{strictly convex} function $\psi:\BB R_+\to\BB R$.
\end{THM}
\DeferProof

Here, the intractable term $P_{\HAT{\BM\beta}}[\BM Y\mid\BM y_T]$ in Eq.~\eqref{eq:q-orig-obj} is replaced with a tractable term $P_{\HAT{\BM\beta}}[\BM Y]$. In this work, we use $\psi(w):=-\log w$, and the objective Eq.~\eqref{eq:q-equiv-obj} instantiates as
\AL{&\min_{\BM\theta}\Exp_{\BM Y\sim P_{\HAT{\BM\beta}}}\Big[{-\log\!\Big(\frac{Q_{\BM\theta}(\BM y_T)[\BM Y]}{P_{\HAT{\BM\beta}}[\BM Y]}\Big)}\Big]
\\\iff{}&\min_{\BM\theta}\Exp_{\BM Y\sim P_{\HAT{\BM\beta}}}[-\log Q_{\BM\theta}(\BM y_T)[\BM Y]+\log P_{\HAT{\BM\beta}}[\BM Y]]
\\\iff{}&\min_{\BM\theta}\Exp_{\BM Y\sim P_{\HAT{\BM\beta}}}[-\log Q_{\BM\theta}(\BM y_T)[\BM Y]]
.\label{eq:q-simp-obj}}
We train the GNN $Q_{\BM\theta}$ using the objective Eq.~\eqref{eq:q-simp-obj}. Notably, since \method{} does not require real diffusion histories as training data, it does not suffer from the scarcity of training data in practice. After training, we use this GNN as the proposal in the M--H MCMC algorithm to estimate the posterior expected hitting times for diffusion history reconstruction. The sampling scheme of the proposal $Q_{\BM\theta}$ is detailed in Appendix~\ref{app:gnn-samp}, which is designed to satisfy the condition $\OP{supp}(Q_{\BM\theta}(\BM y_T))=\OP{supp}(P\mathbin{|}\BM y_T)$ in \THMref{q-equiv-obj}.


\subsection{Complexity Analysis}

\begin{PRP}\label{PRP:cplx}
(i) The time complexity of each iteration of diffusion parameter estimation is $\OP O(T(n+m))$. (ii) The time complexity to sample a history from the proposal is $\OP O(T(n\log n+m))$.
\end{PRP}

\vspace{-0.2em}


According to \PRPref{cplx}, if we optimize $\HAT{\BM\beta}$ for $I$ iterations, optimize $Q_{\BM\theta}$ for $J$ iterations with $K$ samples per iteration, and run MCMC for $S$ iterations with $L$ samples per iteration, then the overall time complexity of \method{} is $\OP O(T(n+m)I+T(n\log n+m)(JK+SL))$. 
If hyperparameters are considered as constants, then the overall time complexity $\OP O(T(n\log n+m))$ is nearly linear w.r.t.\ the output size $\Theta(Tn)$ and the input size $\Theta(n+m)$.


\vspace{-0.3em}

\section{Experiments}\label{sec:exp}

\begin{table}[t]\begin{center}
\caption{Summary of datasets.}\label{tab:data}
\vspace{-8pt}
\begin{scriptsize}
\begin{tabular}{lccccc}
\toprule
\textbf{Dataset}&\textbf{\#Nodes}&\textbf{\#Edges}&\textbf{Timespan}&\textbf{Graph}&\textbf{Diffusion}\\
\midrule
BA      &1,000   &3,984   &10 &Synthetic &Synthetic\\
ER      &1,000   &3,987   &10 &Synthetic &Synthetic\\
\midrule
Oregon2 &11,461  &32,730  &15 &Real      &Synthetic\\
Prost   &15,810  &38,540  &15 &Real      &Synthetic\\
\midrule
BrFarmers &82     &230    &16 &Real      &Real SI  \\
Pol     &18,470  &48,053  &40 &Real      &Real SI  \\
Covid   &344    &2,044   &10 &Real      &Real SIR \\
Hebrew     &3,521   &18,064  &9  &Real      &Real SIR \\
\bottomrule
\end{tabular}
\end{scriptsize}
\end{center}\vspace{-1em}\end{table}

\begin{table*}[h]\centering
\caption{Comparison between estimated $\HAT{\BM\beta}$ and true $\BM\beta$ to justify the mean-field approximation. ``OOM'' indicates ``out of memory.''}\label{tab:exp-beta}
\vspace{-8pt}
\scalebox{0.68}{\begin{tabular}{cl||cc||cc||cc||cc||cc||cc||cc||cc}
\toprule
\multirow{2}*{\textbf{Method}}&\multirow{2}*{\textbf{Training}}&\multicolumn{2}{c||}{BA-SI}&\multicolumn{2}{c||}{ER-SI}&\multicolumn{2}{c||}{Oregon2-SI}&\multicolumn{2}{c||}{Prost-SI}&\multicolumn{2}{c||}{BA-SIR}&\multicolumn{2}{c||}{ER-SIR}&\multicolumn{2}{c||}{Oregon2-SIR}&\multicolumn{2}{c}{Prost-SIR}\\
&&F1$\uparrow$&NRMSE$\downarrow$&F1$\uparrow$&NRMSE$\downarrow$&F1$\uparrow$&NRMSE$\downarrow$&F1$\uparrow$&NRMSE$\downarrow$&F1$\uparrow$&NRMSE$\downarrow$&F1$\uparrow$&NRMSE$\downarrow$&F1$\uparrow$&NRMSE$\downarrow$&F1$\uparrow$&NRMSE$\downarrow$\\
\midrule
\multirow{2}*{GRIN}
&w/ true $\BM\beta$
&.8404&.2123&.8317&.2166&.8320&.2249&.8482&.2155&.7867&.1692&.7626&.2484&.8024&.1651&.8067&.1652\\
&w/ estimated $\HAT{\BM\beta}$
&.8456&.2071&.8324&.2160&.8370&.2199&.8504&.2128&.7833&.1717&.7757&.1939&.8030&.1633&.8068&.1644\\
\midrule
\multirow{2}*{SPIN}
&w/ true $\BM\beta$
&.8414&.2117&.8310&.2167&\multicolumn2{c||}{\multirow2*{OOM}}&\multicolumn2{c||}{\multirow2*{OOM}}&.7832&.1663&.7647&.2321&\multicolumn2{c||}{\multirow2*{OOM}}&\multicolumn2{c}{\multirow2*{OOM}}\\
&w/ estimated $\HAT{\BM\beta}$
&.8477&.2047&.8315&.2170&     &     &     &     &.7869&.1611&.7800&.1909&     &     &     &     \\
\bottomrule
\end{tabular}}
\vspace{-0.3em}
\end{table*}

\begin{table*}[h]\centering
\vspace{-4pt}
\caption{Results for real-world diffusion. ``OOM'' indicates ``out of memory.''}\label{tab:exp-real}
\vspace{-8pt}
\scalebox{0.68}{\begin{tabular}{cl||cc||cc||cc||cc}
\toprule
\multirow{2}*{\textbf{Type}}&\multirow{2}*{\textbf{Method}}&\multicolumn{2}{c||}{BrFarmers}&\multicolumn{2}{c||}{Pol}&\multicolumn{2}{c||}{Covid}&\multicolumn{2}{c}{Hebrew}\\
&&F1$\uparrow$&NRMSE$\downarrow$&F1$\uparrow$&NRMSE$\downarrow$&F1$\uparrow$&NRMSE$\downarrow$&F1$\uparrow$&NRMSE$\downarrow$\\
\midrule
\multirow{5}*{\begin{tabular}{c}Supervised\\(w/ estimated $\HAT{\BM\beta}$)\end{tabular}}
&GCN   &.5409&.6660&.4458&.4946&.3162&.5214&.3350&.6070\\
&GIN   &.4548&.6565&.5203&.4767&.3226&.4951&.3704&.7816\\
&BRITS &.5207&.3995&\multicolumn{2}{c||}{OOM}&.3524&.5333&.3120&.6584\\
&GRIN  &.8003&.2425&.6518&.3731&.5448&.3040&.5916&\BE{.2212}\\
&SPIN  &\BE{.8268}&\BE{.2084}&\multicolumn{2}{c||}{OOM}&\underline{.5917}&\underline{.2932}&.5178&.3330\\
\midrule
\multirow{2}*{MLE}
&DHREC &.6131&.4150&.7023&.3398&.3540&.6023&\underline{.6251}&.4169\\
&CRI   &.6058&.4444&\underline{.7468}&\underline{.2942}&.4170&.5487&.5344&.3552\\
\midrule
Barycenter&\method{} (ours)&\underline{.8206}&\underline{.2142}&\BE{.7471}&\BE{.2903}&\BE{.6240}&\BE{.2637}&\BE{.6411}&\underline{.2983}\\
\bottomrule
\end{tabular}}
\vspace{-0.3em}
\end{table*}

\begin{table*}[h]\centering
\vspace{-4pt}
\caption{Comparison with MLE-based methods on synthetic SI and SIR diffusion. *We use GRIN trained with true $\BM\beta$ as the ideal performance and calculate \emph{Gap} w.r.t.\ this ideal performance.}\label{tab:exp-synth}
\vspace{-8pt}
\scalebox{0.68}{\begin{tabular}{cl||cc|cc||cc|cc||cc|cc||cc|cc}
\toprule
\multirow{2}*{\textbf{Type}}&\multirow{2}*{\textbf{Method}}&\multicolumn{4}{c||}{BA-SI}&\multicolumn{4}{c||}{ER-SI}&\multicolumn{4}{c||}{Oregon2-SI}&\multicolumn{4}{c}{Prost-SI}\\
&&F1$\uparrow$&Gap$\downarrow$&NRMSE$\downarrow$&Gap$\downarrow$&F1$\uparrow$&Gap$\downarrow$&NRMSE$\downarrow$&Gap$\downarrow$&F1$\uparrow$&Gap$\downarrow$&NRMSE$\downarrow$&Gap$\downarrow$&F1$\uparrow$&Gap$\downarrow$&NRMSE$\downarrow$&Gap$\downarrow$\\
\midrule
Ideal&GRIN  &.8404*&---&.2123*&---&.8317*&---&.2166*&---&.8320*&---&.2249*&---&.8482*&---&.2155*&---\\
\midrule
\multirow{2}*{MLE}
&DHREC &.6026&28.30\%&.4644&118.75\%&.6281&24.48\%&.4495&107.53\%&.6038&27.43\%&.4101&82.35\%&.6558&22.68\%&.4138&92.02\%\\
&CRI   &.7502&10.73\%&.3012&41.87\%&.7797&6.25\%&.2744&26.69\%&.8183&1.65\%&.2438&8.40\%&.8083&4.70\%&.2491&15.59\%\\
\midrule[0.3pt]
Barycenter&\method{} (ours)&\BE{.8384}&\BE{0.24\%}&\BE{.2139}&\BE{0.75\%}&\BE{.8269}&\BE{0.58\%}&\BE{.2225}&\BE{2.72\%}&\BE{.8280}&\BE{0.48\%}&\BE{.2289}&\BE{1.78\%}&\BE{.8327}&\BE{1.83\%}&\BE{.2317}&\BE{7.52\%}\\
\midrule[0.68pt]
\multirow{2}*{\textbf{Type}}&\multirow{2}*{\textbf{Method}}&\multicolumn{4}{c||}{BA-SIR}&\multicolumn{4}{c||}{ER-SIR}&\multicolumn{4}{c||}{Oregon2-SIR}&\multicolumn{4}{c}{Prost-SIR}\\
&&F1$\uparrow$&Gap$\downarrow$&NRMSE$\downarrow$&Gap$\downarrow$&F1$\uparrow$&Gap$\downarrow$&NRMSE$\downarrow$&Gap$\downarrow$&F1$\uparrow$&Gap$\downarrow$&NRMSE$\downarrow$&Gap$\downarrow$&F1$\uparrow$&Gap$\downarrow$&NRMSE$\downarrow$&Gap$\downarrow$\\
\midrule
Ideal&GRIN  &.7867*&---&.1692*&---&.7626*&---&.2484*&---&.8024*&---&.1651*&---&.8067*&---&.1652*&---\\
\midrule
\multirow{2}*{MLE}
&DHREC &.5080&35.43\%&.4722&179.08\%&.5500&27.88\%&.4423&78.06\%&.6044&24.68\%&.4478&171.23\%&.6268&22.30\%&.4326&161.86\%\\
&CRI   &.5994&23.81\%&.3356&98.35\%&.6129&19.63\%&.3109&25.16\%&.5761&28.20\%&.3576&116.60\%&.5738&28.87\%&.3406&106.17\%\\
\midrule[0.3pt]
Barycenter&\method{} (ours)&\BE{.7783}&\BE{1.07\%}&\BE{.1633}&\BE{$-$3.49\%}&\BE{.7734}&\BE{$-$1.42\%}&\BE{.1679}&\BE{$-$32.41\%}&\BE{.7928}&\BE{1.20\%}&\BE{.1707}&\BE{3.39\%}&\BE{.7929}&\BE{1.71\%}&\BE{.1690}&\BE{2.30\%}\\
\bottomrule
\end{tabular}}
\vspace{-1em}
\end{table*}

We conduct extensive experiments on both synthetic and real-world datasets to answer the following research questions:
\begin{enumerate}
\renewcommand\labelenumi{\textbf{RQ\theenumi:}}
\item\label{RQ:beta}What is the quality of estimated diffusion parameters $\HAT{\BM\beta}$?
\item\label{RQ:diffus}How does \method{} perform for real-world diffusion?
\item\label{RQ:mle}How does \method{} compare to MLE-based methods?
\item\label{RQ:stab}How stable is \method{} against estimation error of $\HAT{\BM\beta}$? 
\item\label{RQ:scal}How is the scalability of \method{}?
\item\label{RQ:timespan}How does the performance of \method{} vary with timespan?
\item\label{RQ:abla}In M--H MCMC, how does our learned proposal $Q_{\BM\theta}$ compare to a random proposal?
\end{enumerate}

\begin{figure}[t]\centering
\begin{subfigure}[t]{0.51\columnwidth}\centering
\includegraphics[width=\textwidth]{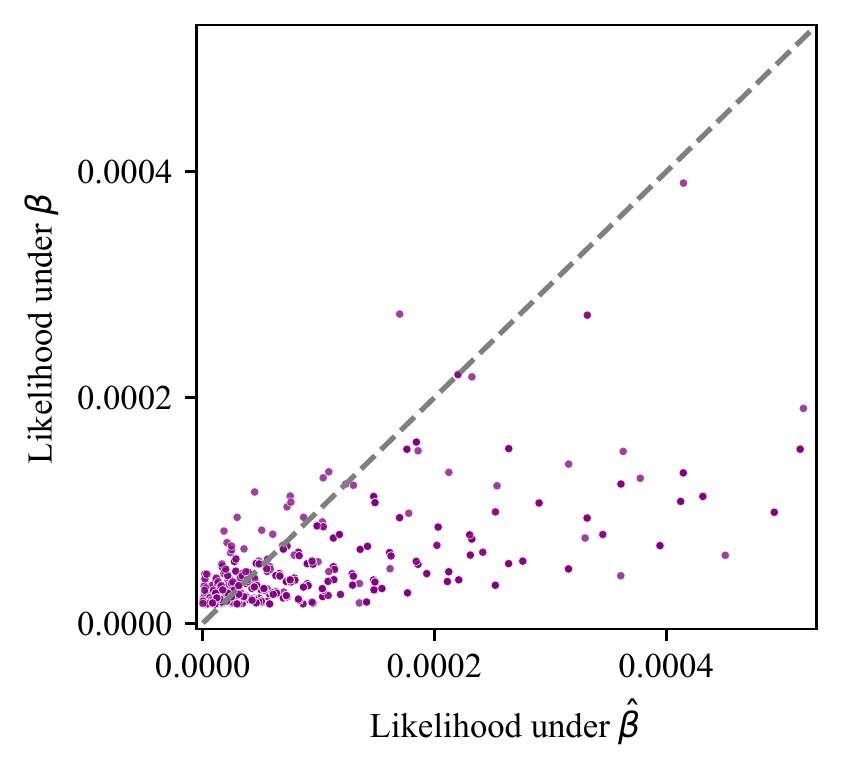}
\vspace{-16pt}
\caption{$P_{\BM\beta}[\BM Y]$ vs $P_{\HAT{\BM\beta}}[\BM Y]$ in the MLE formulation.}
\label{fig:vis-lik}
\end{subfigure}
\hfill
\begin{subfigure}[t]{0.46\columnwidth}\centering
\includegraphics[width=\textwidth]{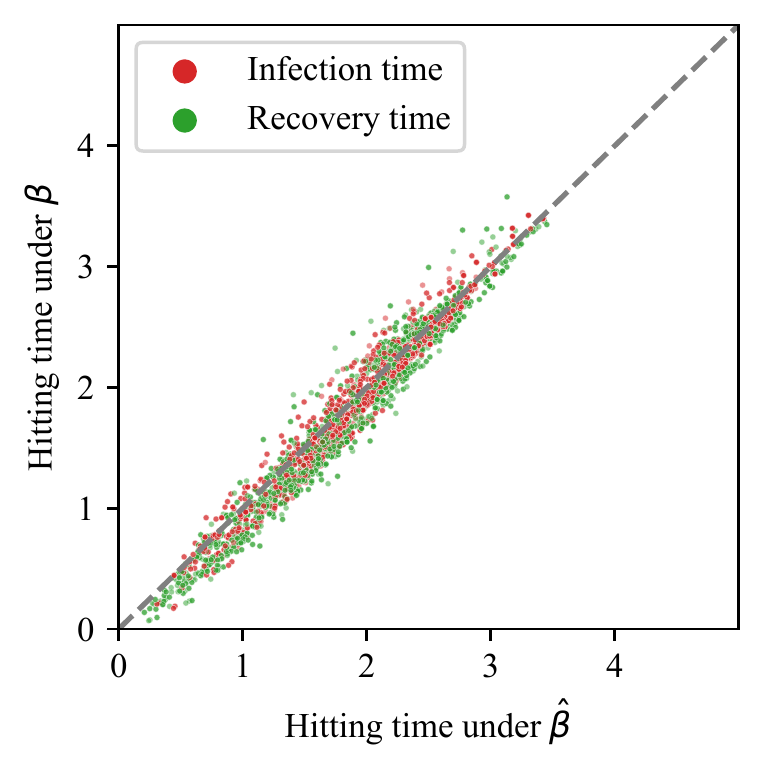}
\vspace{-16pt}
\caption{$\!\!\!\Exp\limits_{P_{\BM\beta}\mid\BM y_T}\!\!\!\![h_u^x(\BM Y)]$ vs $\!\!\!\Exp\limits_{P_{\HAT{\BM\beta}}\mid\BM y_T}\!\!\!\![h_u^x(\BM Y)]$ in the barycenter formulation.}
\label{fig:vis-hit}
\end{subfigure}
\vspace{-8pt}
\caption{Sensitivity of the MLE formulation vs stability of the barycenter formulation.}
\label{fig:vis-stab}
\end{figure}

\begin{figure}[t]\centering
\vspace{-6pt}
\begin{subfigure}[t]{0.465\columnwidth}\centering
\includegraphics[width=\textwidth]{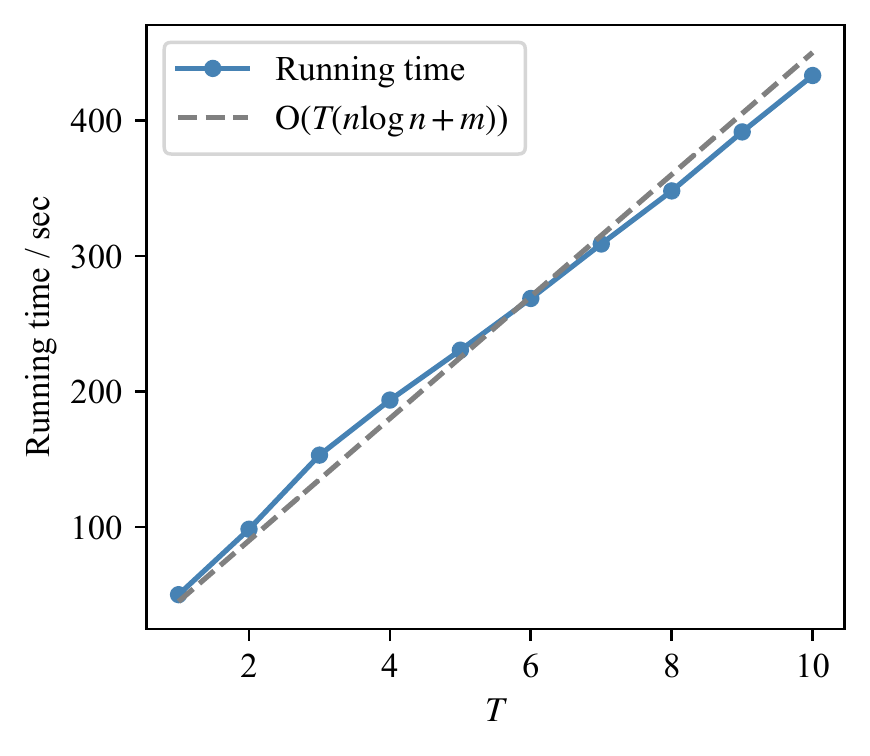}
\vspace{-16pt}
\caption{Running time vs $T$.}
\label{fig:scal-T}
\end{subfigure}
\hfill
\begin{subfigure}[t]{0.5\columnwidth}\centering
\includegraphics[width=\textwidth]{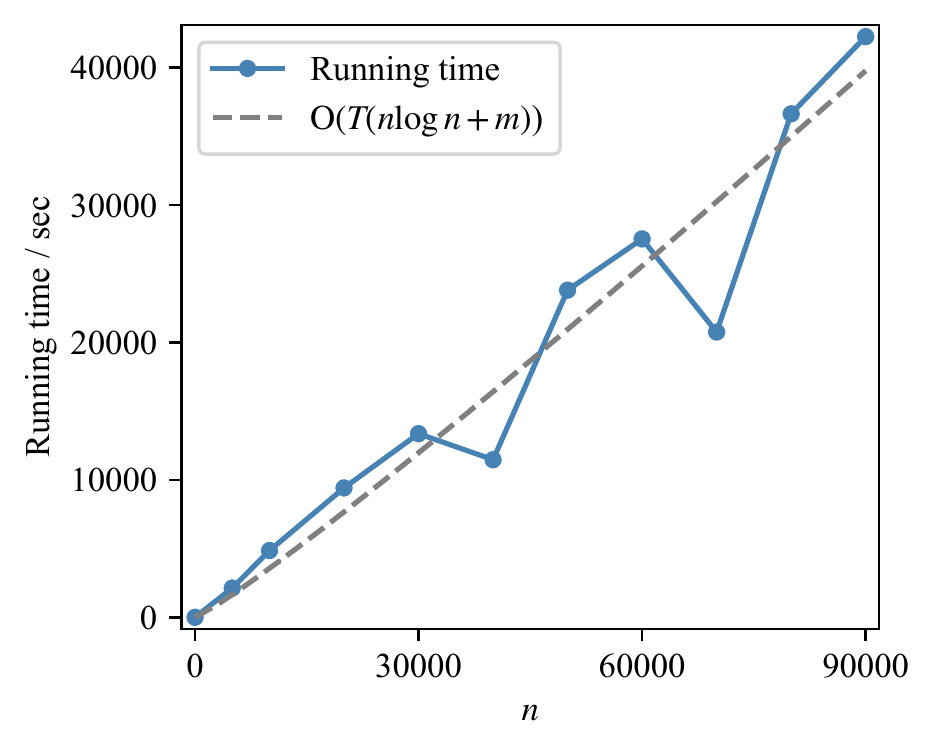}
\vspace{-16pt}
\caption{Running time vs $n$.}
\label{fig:scal-n}
\end{subfigure}
\vspace{-8pt}
\caption{Running time (training time + testing time).}
\label{fig:scal}
\end{figure}

\vspace{-0.5em}

\subsection{Experimental Setting}

\subsubsection{Datasets} 
We use 3 types of datasets. 
\textbf{(D1) Synthetic graphs and diffusion:} Barab\' asi--Albert (BA) \cite{barabasi-albert} and  Erd\H os--R\' enyi (ER) \cite{erdos-renyi} random graphs with synthetic SI and SIR diffusion. 
\textbf{(D2) Synthetic diffusion on real graphs:} Oregon2 \cite{oregon2} and Prost \cite{data-prost} with synthetic SI and SIR diffusion. 
\textbf{(D3) Real diffusion on real graphs:} BrFarmers \cite{data-farmers-orig,data-farmers} and Pol \cite{data-pol} with SI-like real diffusion, and Covid and Hebrew \cite{data-heb} with SIR-like real diffusion. 
Please refer to Appendix~\ref{app:exp-data} for detailed description of datasets. 

\subsubsection{Baselines} We consider 
2 types of baselines. 
\textbf{(B1) Supervised methods for time series imputation:} DASH can be alternatively formulated as time series imputation, so we compare \method{} with latest imputation methods including GCN \cite{gcn}, \cite{gin}, BRITS \cite{brits}, GRIN \cite{grin}, and SPIN \cite{spin}.
\textbf{(B2) MLE-based methods for diffusion history reconstruction:} 
DHREC \cite{pcvc} and CRI \cite{cri}.
Please refer to Appendix~\ref{app:exp-bl} for description of baselines.

\subsubsection{Evaluation Metrics} To measure the similarity between the true history and the reconstructed history, we use the macro F1 score (F1; the higher, the better) and the normalized rooted mean squared error (NRMSE; the lower, the better) of hitting times, where
\AL{\OP{NRMSE}(\BM Y,\HAT{\BM Y}):=\sqrt{\tfrac{\!\sum\limits_{u\in\CAL V}\!\!((h_u^\SIRI(\BM Y)\!-\!h_u^\SIRI(\HAT{\BM Y}))^2\!+\!(h_u^\SIRR(\BM Y)\!-\!h_u^\SIRR(\HAT{\BM Y}))^2)}{2n(T+1)^2}}.}
In Sec.~\ref{sec:exp-mle}, we also use the performance gap to the ideal performance as a metric. The smaller gap, the better. Let $s$ and $s^*$ denote the actual performance and the ideal performance, respectively. For F1, $\OP{Gap}(s,s^*):=(s^*-s)/s^*$. For NRMSE, $\OP{Gap}(s,s^*):=(s-s^*)/s^*$.

\subsubsection{Reproducibility}
Please refer to Appendix~\ref{app:repro}.

\subsection{Quality of Estimated Diffusion Parameters}\label{sec:exp-beta}
To answer \RQref{beta}, 
we compare the performance of supervised methods under true $\BM\beta$ with their performance under estimated $\HAT{\BM\beta}$. We use two strongest imputation methods GRIN and SPIN. Since true diffusion parameters of real diffusion are not available, we only use datasets D1 and D2 in this experiment. 
Results are shown in Table~\ref{tab:exp-beta}. Whether trained with true $\BM\beta$ or estimated $\HAT{\BM\beta}$, the performance has no significant difference. 
Results suggest that mean-field approximation is sufficient to estimate diffusion parameters accurately, and the estimated diffusion parameters can help supervised methods achieve strong performance when the diffusion model is known.

\subsection{Performance for Real-World Diffusion}
For real-world diffusion, the diffusion model is not exactly SI/SIR, and true diffusion parameters are unknown. Hence, it is important to test how \method{} generalizes from SI/SIR to real-world diffusion. To answer \RQref{diffus} and \RQref{mle}, we comprehensively compare \method{} with (B1) supervised methods for time series imputation and (B2) MLE-based methods for diffusion history reconstruction on real-world diffusion in D3. Since true diffusion parameters of real diffusion are not available, we estimate diffusion parameters by \method. For B1, we use estimated diffusion parameters to generate training data. For B2, we feed estimated diffusion parameters to MLE-based methods.

The results for real diffusion are shown in Table~\ref{tab:exp-real}. \method{} generalizes well to real diffusion and consistently achieves strong performance on all datasets. For instance, \method{} is 10.06\% better in NRMSE than the best baseline for the Covid dataset. In contrast, the performance of supervised methods degrades drastically when real diffusion deviates from SI/SIR models. This is because the training data generated by the SI/SIR model follow a different distribution from the real diffusion. Only for BrFarmers do supervised methods achieve good performance, because the diffusion in BrFarmers is very close to the SI model \cite{data-farmers}. 
For MLE-based methods, their performance varies largely across datasets. 
This is because real diffusion may not be close to the SI/SIR model, so the likelihood as their objective function may fail.



\subsection{Comparison with MLE-Based Methods}\label{sec:exp-mle}
To further answer \RQref{mle}, we compare \method{} to MLE-based methods also with synthetic diffusion on both synthetic graphs in D1 and real graphs in D2. 
Note that here we do not directly compare with supervised methods for the following reasons. 
(1) Table~\ref{tab:exp-beta} shows that if supervised methods know the diffusion model of 
test data, then they can generate training data that follow 
the same distribution as the test data. Thus, they are expected to perform well. 
(2) Meanwhile, Table~\ref{tab:exp-real} shows 
the superiority of supervised methods comes only from knowing the underlying diffusion model (including its true parameters), 
which is almost impossible in practice due to the scarcity of training data. As long as they 
have no access to the true diffusion model, their performance drop drastically. Therefore, it is meaningless to compare with supervised methods for synthetic diffusion. Instead, we train GRIN with true diffusion parameters and use its results as the \emph{ideal} performance. Then for MLE-based methods and \method{}, we compare their performance gaps to this ideal performance. We do not use SPIN here because it has similar performance with GRIN, while SPIN is out of memory on D2.

The results are shown in Table~\ref{tab:exp-synth}. \method{} consistently achieves the strongest performance and significantly outperforms state-of-the-art MLE-based methods for all datasets. Notably, \method{} even outperforms GRIN for BA-SIR and ER-SIR. In contrast, the performance of MLE-based methods vary largely due to their instability to estimation error of diffusion parameters. For instance, \method{} has only 1.07\% gap in F1 for BA-SIR, while MLE-based methods have at least 23.81\% gap in F1. Results demonstrate the superior performance of \method{} over state-of-the-art MLE-based methods.

\vspace{-0.4em}

\subsection{Additional Experiments}

\subsubsection{Stability against Estimation Error of Diffusion Parameters.}
To answer \RQref{stab} and demonstrate the stability of our barycenter formulation, we visualize the likelihoods of histories and the posterior expected hitting times under true and estimated diffusion parameters. Since the number of possible histories under the SIR model is roughly $\OP O\!\big({\binom{T+3}2}^{\!n}\big)$, it is intractable to compute them on large graphs. Thus, we use a graph with $n=6$ and $T=4$ so that likelihoods and posterior expected hitting times can be computed exactly. We visualize them under the SIR model with $\BM\beta=(0.3,0.2)\Tp$ and $\HAT{\BM\beta}=(0.2,0.3)\Tp$. Fig.~\ref{fig:vis-stab} displays the results. Fig.~\ref{fig:vis-lik} shows that the likelihoods of histories change drastically in the presence of parameter estimation error. In contrast, Fig.~\ref{fig:vis-hit} shows that posterior expected hitting times are almost identical under $\BM\beta$ and $\HAT{\BM\beta}$. Therefore, our barycenter formulation is more stable than the MLE formulation against estimation error of diffusion parameters, which agrees with our theoretical analyses in \THMref{diffus-mle-sens} and \THMref{hit-stable}.

\vspace{-0.2em}

\subsubsection{Scalability}
To answer \RQref{scal}, we evaluate the scalability of \method{} by varying $T$ and $n$. We generate BA graphs with attachment $4$ to obtain scale-free networks with various sizes. Fig.~\ref{fig:scal-T} shows running times under $n=1,000$ and $T=1,\dots,10$, and Fig.~\ref{fig:scal-n} shows running times under $T=10$ and $n$ up to 90k. Results demonstrate that the running times of \method{} scale near-linearly w.r.t.\ $T$ and $n$, which agrees with its time complexity $\OP O(T(n\log n+m))$.

\vspace{-0.2em}

\subsubsection{Effect of Timespan \& Ablation Study}
In Appendix~\ref{app:exp}, we answer \RQref{timespan} and \RQref{abla}. 
Appendix~\ref{app:exp-timespan} compares \method{} and MLE-based methods under various timespans. It demonstrates that \method{} can better handle the higher uncertainty induced by larger timespan than MLE-based methods. Appendix~\ref{app:exp-abla} is an ablation study on the effect of the number of training steps. It shows that the learned proposal performs better than untrained proposal.


\section{Related Work}


Diffusion on graphs are deterministic or stochastic processes where information or entities on nodes transmit through edges \cite{kermack1927contribution,granovetter1978threshold,kempe2003maximizing,aron1984seasonality,watts2002simple,ruan2015kinetics,lee2017universal,torok2017cascading}. 
In this section, we review related work on graph diffusion, which can be grouped into 
forward 
and inverse problems.

\vskip0.5em\noindent\textbf{Forward problems on graph diffusion.}
\label{sec:related-diffus}
The vast majority of research on diffusion or dynamic graphs \cite{DBLP:conf/sigir/FuH21,DBLP:conf/kdd/FuFMTH22} are devoted to forward problems. 
Pioneering works derive 
epidemic thresholds for random graphs from probabilistic perspectives \cite{bikhchandani1992theory} or for arbitrary graphs from spectral perspectives \cite{wang2003epidemic,ganesh2005effect,valler2011epidemic,prakash2012threshold}. Later, observational studies investigate influence patterns of diffusion processes \cite{goldenberg2001talk,briesemeister2003epidemic,leskovec2007dynamics,habiba2011working}. On the algorithmic side, researchers have made tremendous effort to diffusion-related optimization problems, such as influence maximization \cite{goldenberg2001talk,richardson2002mining,kempe2003maximizing,chen2009efficient,datta2010viral,chen2010scalable,chen2011influence,goyal2011data} and node immunization \cite{hayashi2004oscillatory,kempe2005influential,tong2010vulnerability,prakash2010virus,tuli2012blocking}. Recently, differential equations of graph diffusion has also been applied to the design of graph convolutional networks to alleviate oversmoothing \cite{klicpera2019diffusion,chamberlain2021grand,zhao2021adaptive,wang2021dissecting,chen2022optimization,thorpe2022grand++}. 

\vskip0.5em\noindent\textbf{Inverse problems on graph diffusion.}
\label{sec:related-inverse}
Compared with forward problems on graph diffusion, the inverse problems 
are in general more difficult 
due to the challenge of ill-posedness. 
The inverse problems split into two categories by whether diffusion histories are known. In the one category where diffusion histories are known, the problems are relatively more tractable because the search space is smaller. These problems include estimating diffusion parameters \cite{gruhl2004information,myers2010convexity,gomez2012inferring,zong2012inferring,gardner2014inferring,song2017history}, recovering graph topology \cite{gruhl2004information,myers2010convexity,gomez2012inferring,zong2012inferring,gardner2014inferring}, and inferring diffusion paths \cite{abrahao2013trace,fajardo2013inferring,rozenshtein2016reconstructing,song2017history,sun2017collaborative,dawkins2021diffusion}.

The other category where diffusion histories are unknown is much less studied. In this category, the problems are often harder because the number of possible histories is explosively large. Among them, most works focus on the \emph{diffusion source localization} problem. 
Only recently has research emerged on the even harder problem \emph{diffusion history reconstruction}. 
\emph{(1) Diffusion source localization.} The source localization problem aims to find the source nodes of diffusion. Early works focus on designing localization algorithms 
based on graph theory and network science \cite{lappas2010finding,shah2011rumors,shah2012rumor,farajtabar2015back,zhu2016locating,zhu2016information,zhu2017catch,feizi2018network,ri,lpsi}. These methods may not generalize well to various diffusion models. Later works propose data-driven methods that utilize graph neural networks to learn to identify sources from data \cite{gcnsi,ivgd,slvae}. 
\emph{(2) Diffusion history reconstruction.} Compared with source localization, diffusion history reconstruction is even harder because the search space of possible histories is larger. Existing methods for diffusion history reconstruction are exclusively based on the MLE formulation, including DHREC \cite{pcvc}, CRI \cite{cri}, and SSR \cite{ssr}. These methods 
assume that true diffusion parameters \cite{pcvc,ssr} and/or partial diffusion histories are known \cite{ssr}, or cannot reconstruct a complete diffusion history \cite{cri}. 
Meanwhile, diffusion history reconstruction can be alternatively formulated as a 
time series imputation problem. 
State-of-the-art methods include BRITS \cite{brits} for multivariate time series, and 
GRIN \cite{grin} and SPIN \cite{spin} for graph time series. They 
are all supervised and thus suffer from the scarcity of training data of real diffusion histories. Furthermore, since the true diffusion model is unknown for real-world diffusion, it is difficult to synthesize training data that follow the same distribution as the true diffusion model. Therefore, they have limited applicability 
in practice.



\section{Conclusion}
In this work, we have studied a challenging problem: reconstructing diffusion history from a single snapshot. 
To address the sensitivity of the MLE formulation,
we have proposed a barycenter formulation that is provably stable against the estimation error of diffusion parameters. 
We have further developed an effective solver named \method{} for the barycenter formulation, which is based on Metropolis--Hastings MCMC with a learned optimal proposal. Our method is unsupervised, which is desirable in practice due to the scarcity of training data. Extensive experiments have shown that \method{} consistently achieve strong performance for both synthetic and real-world diffusion.


\begin{acks}
This work was supported in part by NSF (1947135 and 
2134079), 
the NSF Program on Fairness in AI in collaboration with Amazon (1939725), 
DARPA (HR001121C0165), 
NIFA (2020-67021-32799), 
DHS (17STQAC00001-06-00), 
ARO (W911NF2110088), 
C3.ai Digital Transformation Institute, 
and IBM-Illinois Discovery Accelerator Institute. 
The work of Lei Ying was supported in part by NSF (2134081).
The content of the information in this document does not necessarily reflect the position or the policy of the Government or Amazon, and no official endorsement should be inferred. The U.S.\ Government is authorized to reproduce and distribute reprints for Government purposes notwithstanding any copyright notation here on.
\end{acks}

\bibliographystyle{ACM-Reference-Format}
\balance
\bibliography{main}

\clearpage
\appendix
\makeatletter
\newcommand{\LABEL}[2]{\protected@write\@auxout{}{\string\newlabel{#1}{{#2}{\thepage}{#2}{#1}{}}}\hypertarget{#1}{}}
\makeatother
\newcommand\LIST[1]{\begin{itemize}#1\end{itemize}}
\newcommand\AITEM[1]{\item[\ref{#1}]\ref{#1::name}\dotfill\pageref{#1}}
\newcommand\AOBJ[3]{\LABEL{#3::name}{#2}#1{#2}\label{#3}}
\newcommand\ASEC[2]{\AOBJ{\section}{#1}{#2}}
\newcommand\ASSEC[2]{\AOBJ{\subsection}{#1}{#2}}

\section*{Contents}
\LIST{
    \AITEM{app:sir}
    \AITEM{app:pf}
    \LIST{
        \AITEM{app:pf-1}
        \AITEM{app:pf-2}
        \AITEM{app:pf-3}
        \AITEM{app:pf-4}
        \AITEM{app:pf-5}
        \AITEM{app:pf-6}
    }
    \AITEM{app:gnn}
    \LIST{
        \AITEM{app:gnn-arch}
        \AITEM{app:gnn-samp}
    }
    \AITEM{app:setting}
    \LIST{
        \AITEM{app:exp-data}
        \AITEM{app:exp-bl}
        \AITEM{app:repro}
    }
    \AITEM{app:exp}
    \LIST{
        \AITEM{app:exp-timespan}
        \AITEM{app:exp-abla}
    }
    \AITEM{app:concl}
}


\ASEC{Preliminaries on the SIR Model}{app:sir}
In this section, we introduce the detailed definition of the SIR diffusion model.
According to the Markov property, the probability of a history $\BM Y$ can be factorized in the temporal order:
\AL{P_{\BM\beta}[\BM Y]=P[\BM y_0]\prod_{t=0}^{T-1}P_{\BM\beta}[\BM y_{t+1}\mid\BM y_t].}
The initial distribution $P[\BM y_0]$ is not defined in the SIR model, so it does not depend on diffusion parameters $\BM\beta$. For the transition probabilities $P_{\BM\beta}[\BM y_{t+1}\mid\BM y_t]$, they can be further factorized into the transition probability of each single node because every node is assumed to be independent with other nodes at the same time:
\AL{P_{\BM\beta}[\BM y_{t+1}\mid\BM y_t]=\prod_{u\in\CAL V}P_{\BM\beta}[y_{t+1,u}\mid\BM y_t].}
If a node $u$ is susceptible at time $t+1$, then it has to be susceptible at time $t$, and all its infected neighbors failed to infect it:
\AL{
P_{\BM\beta}[y_{t+1,u}=\SIRS\mid\BM y_t]:=\begin{cases}\prod\limits_{v\in\CAL N_u\wedge y_{t,v}=\SIRI}\!\!\!\!\!\!\!\!(1-\beta^\SIRI),&\text{if }y_{t,u}=\SIRS;\\
0,&\text{if }y_{t,u}=\SIRI\text{ or }\SIRR.\end{cases}
}
If a node $u$ is infected at time $t+1$, then either it is infected by its infected neighbors and does not recover immediately, or it is already infected and has not recovered yet:
\AL{
P_{\BM\beta}[y_{t+1,u}=\SIRI\mid\BM y_t]:=\begin{cases}
\Big(1-\!\!\!\!\!\!\!\!\prod\limits_{v\in\CAL N_u\wedge y_{t,v}=\SIRI}\!\!\!\!\!\!\!\!(1-\beta^\SIRI)\bigg)(1-\beta^\SIRR),&\text{if }y_{t,u}=\SIRS;\\
1-\beta^\SIRR,&\text{if }y_{t,u}=\SIRI;\\
0,&\text{if }y_{t,u}=\SIRR.\end{cases}
}
If a node $u$ is recovered at time $t+1$, then either it recovers just at time $t+1$, or it is already recovered previously:
\AL{
P_{\BM\beta}[y_{t+1,u}=\SIRR\mid\BM y_t]:=\begin{cases}
\bigg(1-\!\!\!\!\!\!\!\!\prod\limits_{v\in\CAL N_u\wedge y_{t,v}=\SIRI}\!\!\!\!\!\!\!\!(1-\beta^\SIRI)\bigg)\beta^\SIRR,&\text{if }y_{t,u}=\SIRS;\\
\beta^\SIRR,&\text{if }y_{t,u}=\SIRI;\\
1,&\text{if }y_{t,u}=\SIRR.\end{cases}
}

\ASEC{Proofs}{app:pf}

\ASSEC{Proof of \THMref{diffus-prob-nphard}}{app:pf-1}

The precise definition of approximating the probability of a snapshot is stated in \PRBref{diffus-prob}.

\begin{PRB}[Approximating the probability of a snapshot]\label{PRB:diffus-prob}
Under the SIR model, given a graph $(\CAL V,\CAL E)$, diffusion parameters $\BM\beta$, a timespan $T$, a snapshot $\BM y_T$, an initial distribution $P[\BM y_0]$, and a relative error tolerance $0<\epsilon<1$, find a number $p$ such that
\AL{(1-\epsilon)P_{\BM\beta}[\BM y_T]<p<(1+\epsilon)P_{\BM\beta}[\BM y_T].}
\end{PRB}

Now we prove \THMref{diffus-prob-nphard}.

\begin{proof}[Proof of \THMref{diffus-prob-nphard}]
By reduction from the Minimum Dominating Set (MDS) problem. Suppose that we are to find the minimum dominating set of a graph $(\CAL V,\CAL E)$, where $|\CAL V|=n$. We will construct an instance of \PRBref{diffus-prob} that can be utilized to solve the MDS problem.

The graph for \PRBref{diffus-prob} is the same graph $(\CAL V,\CAL E)$. We choose the diffusion parameters $\beta^\SIRI:=1$ and $\beta^\SIRR:=0$, and choose the timespan $T=1$. We consider the snapshot $\BM y_1$ to be $y_{1,u}:=\SIRI$ for all nodes $u\in\CAL V$. Pick a relative error tolerance $0<\epsilon<1$ arbitrarily. Initially, we define every node to be independently infected with probability $\frac1{1+\frac{1+\epsilon}{1-\epsilon}2^n}$:
\AL{P[\BM y_0]:=\bigg(\frac1{1+\frac{1+\epsilon}{1-\epsilon}2^n}\bigg)^{\!n^\SIRI(\BM y_0)}\bigg(1-\frac1{1+\frac{1+\epsilon}{1-\epsilon}2^n}\bigg)^{\!n-n^\SIRI(\BM y_0)}.}
Then run the oracle for \PRBref{diffus-prob} to get the output number $p$, which satisfies
\AL{(1-\epsilon)P_{\BM\beta}[\BM y_1]<p<(1+\epsilon)P_{\BM\beta}[\BM y_1].}
We claim that the minimum dominating set is of size $s$ iff the output $p$ satisfies
\AL{\frac{(1-\epsilon)\big(\frac{1+\epsilon}{1-\epsilon}2^n\big)^{\!n-s}}{\big(1+\frac{1+\epsilon}{1-\epsilon}2^n\big)^{\!n}}<p<\frac{(1-\epsilon)\big(\frac{1+\epsilon}{1-\epsilon}2^n\big)^{\!n-s+1}}{\big(1+\frac{1+\epsilon}{1-\epsilon}2^n\big)^{\!n}}.\label{eq:mds-itv}}
Since the intervals in Eq.~\eqref{eq:mds-itv} have no overlap for different $s$, then we can uniquely determine the minimum size $s$ from the output $p$.

To prove the claim, note that $\beta^\SIRI:=1$ implies that $P_{\BM\beta}[\BM y_1\mid\BM y_0]>0$ iff the infected nodes in $\BM y_0$ is a dominating set. Let $s$ denote the size of the minimum dominating set and $c_k$ denote the number of dominating sets of size $k$. Hence,
\AL{P_{\BM\beta}[\BM y_1]=\sum_{k=s}^nc_k\bigg(\frac1{1+\frac{1+\epsilon}{1-\epsilon}2^n}\bigg)^{\!k}\bigg(1-\frac1{1+\frac{1+\epsilon}{1-\epsilon}2^n}\bigg)^{\!n-k}.}
Since $s$ is the size of the minimum dominate set, then we have $c_s\ge1$. Thus,
\AL{P_{\BM\beta}[\BM y_1]&\ge c_s\bigg(\frac1{1+\frac{1+\epsilon}{1-\epsilon}2^n}\bigg)^{\!s}\bigg(1-\frac1{1+\frac{1+\epsilon}{1-\epsilon}2^n}\bigg)^{\!n-s}
\\&\ge\bigg(\frac1{1+\frac{1+\epsilon}{1-\epsilon}2^n}\bigg)^{\!s}\bigg(1-\frac1{1+\frac{1+\epsilon}{1-\epsilon}2^n}\bigg)^{\!n-s}.}
Hence,
\AL{
p&>(1-\epsilon)P_{\BM\beta}[\BM y_1]
\\&\ge(1-\epsilon)\bigg(\frac1{1+\frac{1+\epsilon}{1-\epsilon}2^n}\bigg)^{\!s}\bigg(1-\frac1{1+\frac{1+\epsilon}{1-\epsilon}2^n}\bigg)^{\!n-s}
\\&=\frac{(1-\epsilon)\big(\frac{1+\epsilon}{1-\epsilon}2^n\big)^{\!n-s}}{\big(1+\frac{1+\epsilon}{1-\epsilon}2^n\big)^{\!n}}
.\label{eq:mds-lb}
}
Furthermore, since $c_k\le\binom nk$ and $\Big(\frac1{1+\frac{1+\epsilon}{1-\epsilon}2^n}\Big)\Big(1-\frac1{1+\frac{1+\epsilon}{1-\epsilon}2^n}\Big)^{\!-1}=\frac1{\frac{1+\epsilon}{1-\epsilon}2^n}<1$, then:
\AL{P_{\BM\beta}[\BM y_1]&\le\sum_{k=s}^n\binom nk\bigg(\frac1{1+\frac{1+\epsilon}{1-\epsilon}2^n}\bigg)^{\!s}\bigg(1-\frac1{1+\frac{1+\epsilon}{1-\epsilon}2^n}\bigg)^{\!n-s}
\\&\le\bigg(\sum_{k=s}^n\binom nk\bigg)\bigg(\frac1{1+\frac{1+\epsilon}{1-\epsilon}2^n}\bigg)^{\!s}\bigg(1-\frac1{1+\frac{1+\epsilon}{1-\epsilon}2^n}\bigg)^{\!n-s}
\\&\le\bigg(\sum_{k=0}^n\binom nk\bigg)\bigg(\frac1{1+\frac{1+\epsilon}{1-\epsilon}2^n}\bigg)^{\!s}\bigg(1-\frac1{1+\frac{1+\epsilon}{1-\epsilon}2^n}\bigg)^{\!n-s}
\\&=2^n\bigg(\frac1{1+\frac{1+\epsilon}{1-\epsilon}2^n}\bigg)^{\!s}\bigg(1-\frac1{1+\frac{1+\epsilon}{1-\epsilon}2^n}\bigg)^{\!n-s}
.}
Hence,
\AL{
p&<(1+\epsilon)P_{\BM\beta}[\BM y_1]
\\&\le(1+\epsilon)\cdot2^n\bigg(\frac1{1+\frac{1+\epsilon}{1-\epsilon}2^n}\bigg)^{\!s}\bigg(1-\frac1{1+\frac{1+\epsilon}{1-\epsilon}2^n}\bigg)^{\!n-s}
\\&=\frac{(1-\epsilon)\big(\frac{1+\epsilon}{1-\epsilon}2^n\big)^{\!n-s+1}}{\big(1+\frac{1+\epsilon}{1-\epsilon}2^n\big)^{\!n}}
.\label{eq:mds-ub}}
Combining Eq.~\eqref{eq:mds-lb} and Eq.~\eqref{eq:mds-ub} yields our claim Eq.~\eqref{eq:mds-itv}.

The numbers involved can be stored in $\Poly\!\big(n,\log\frac1\epsilon\big)$ bits and be computed using high-precision arithmetics within $\Poly\!\big(n,\log\frac1\epsilon\big)$ time. Therefore, this gives a polynomial-time reduction from the MDS problem to \PRBref{diffus-prob}, so the NP-hardness of the MDS problem \cite{karp-npc} implies that \PRBref{diffus-prob} is NP-hard.
\end{proof}

\ASSEC{Proof of \THMref{diffus-par-mle-nphard}}{app:pf-2}
The precise definition of diffusion parameter MLE is stated in \PRBref{diffus-par-mle}.

\begin{PRB}[Diffusion parameter MLE]\label{PRB:diffus-par-mle}
Under the SIR model, given a graph $(\CAL V,\CAL E)$, a timespan $T$, a snapshot $\BM y_T$, an initial distribution $P[\BM y_0]$, and a relative error tolerance $0<\epsilon<1$, find $\HAT{\BM\beta}$ where
\AL{\exists\BM\beta\in\argmax_{\BM\beta}P_{\BM\beta}[\BM y_T]:(1-\epsilon)\BM\beta<\HAT{\BM\beta}<(1+\epsilon)\BM\beta.}
\end{PRB}

Before proving \THMref{diffus-par-mle-nphard}, we give a technical lemma.

\begin{LEM}\label{LEM:bern-ub}
For $r\ge1$, $c>0$, and $0\le x\le\frac c{r^2}$,
\AL{(1+x)^r\le1+\Big(r+\frac{\RM e^{c}\!-1}{2}\Big)x.}
In particular, for $r\ge1$ and $0\le x\le\frac1{r^2}$,
\AL{(1+x)^r\le1+\Big(r+\frac{\RM e-1}2\Big)x\le1+(r+1)x.}
\end{LEM}
\begin{proof}
Define an auxiliary function:
\AL{\phi(z):=1+cz+\frac{\RM e^{c}-1}{2}cz^2-\RM e^{cz}.}
Its first order derivative $\phi'(z)=c+(\RM e^{c}-1)cz-c\RM e^{cz}$ is concave, so for every $0\le z\le1$,
\AL{\phi'\!(z)\ge\min\{\phi'\!(0),\,\phi'\!(1)\}=\min\{0,\,0\}=0.}
This suggests that $\phi(z)$ is increasing over $0\le z\le 1$. Therefore, since $1+\big(r+\frac{\RM e^{c}\!-1}{2}\big)x-(1+x)^r$ is concave w.r.t.\ $x\ge0$ for $r\ge1$, then for every $0\le x\le\frac c{r^2}$,
\AL{
&1+\Big(r+\frac{\RM e^{c}\!-1}{2}\Big)x-(1+x)^r
\\\ge{}&\min\!\Big\{1+\Big(r+\frac{\RM e^{c}\!-1}{2}\Big)0-(1+0)^r,\\
&\qquad\ 1+\Big(r+\frac{\RM e^{c}\!-1}{2}\Big)\frac c{r^2}-\Big(1+\frac c{r^2}\Big)^{\!r}\Big\}
\\={}&\min\!\Big\{0,\,1+\frac cr+\frac{\RM e^{c}-1}{2}\frac c{r^2}-\Big(1+\frac c{r^2}\Big)^{\!r}\Big\}
\\\ge{}&\min\!\Big\{0,\,1+\frac cr+\frac{\RM e^{c}-1}{2}\frac c{r^2}-(\RM e^{c/r^2})^{r}\Big\}
\\={}&\min\!\Big\{0,\,\phi\Big(\frac1r\Big)\Big\}\ge\min\{0,\,\phi(0)\}=\min\{0,0\}=0
.}
\end{proof}


Now we are ready to prove \THMref{diffus-par-mle-nphard}.


\begin{proof}[Proof of \THMref{diffus-par-mle-nphard}]
By reduction from the Minimum Dominating Set (MDS) problem. Suppose that we are to find the minimum dominating set of a graph $(\CAL V,\CAL E)$, where $|\CAL V|=n$ and $\CAL E$ contains no self-loops. If $|\CAL E|=0$, then the only dominating set is $\CAL V$. Thus, we can assume that $|\CAL E|\ge1$ from now on, which implies $n\ge2$ and that the size of the minimum dominating set is at most $n-1$. We will construct an instance of \PRBref{diffus-par-mle} that can be utilized to solve the MDS problem.

We create two auxiliary vertices $\CAL V':=\{a,b\}$ and create edges between them and all nodes in $\CAL V$, i.e., the graph for \PRBref{diffus-par-mle} is $(\CAL V\cup\CAL V',\,\CAL E\cup(\CAL V\times\CAL V'))$. We choose timespan $T=1$. We consider the snapshot $\BM y_1$ to be $y_{1,a}=\SIRS$, $y_{1,b}=\SIRR$, and $y_{1,u}=\SIRR$ for all nodes $u\in\CAL V$. We choose the relative error tolerance
\AL{\epsilon:=\min\left\{\frac1{16(n+1)n8^n-1},\,\frac{\frac1{\sqrt[n]2}-\frac1{\sqrt[n-1]2}\big(1+\frac1{n2^n}\big)^{\!\frac1{n-1}}}{1-\frac1{\sqrt[n-1]2}\big(1+\frac1{n2^n}\big)^{\!\frac1{n-1}}}\right\}.}
We define the initial distribution as
\AL{P[\BM y_0]\propto\frac1{n^\SIRI(\BM y_0)}\Big(\frac1{2n4^n}\Big)^{\!n^\SIRI(\BM y_0)}}
iff $y_{0,a}=y_{0,b}=\SIRS$, $\{u\in\CAL V:y_{0,u}=\SIRI\}$ is a dominating set of $(\CAL V,\CAL E)$, and $y_{0,u}=\SIRR$ for all other nodes; otherwise, $P[\BM y_0]:=0$. Then run the oracle for \PRBref{diffus-par-mle} to get diffusion parameter estimates $\HAT{\BM\beta}=[\HAT\beta^\SIRI,\HAT\beta^\SIRR]\Tp$. We claim that the minimum dominating set is of size $s$ iff the output $\HAT\beta^\SIRI$ satisfies
\AL{1-\frac1{\sqrt[s+1]2}<\HAT\beta^\SIRI<1-\frac1{\sqrt[s]2}.\label{eq:diffus-par-claim}}
Since the intervals in Eq.~\eqref{eq:diffus-par-claim} have no overlap for different $s$, then we can uniquely determine the minimum size $s$ from the output $\HAT{\BM\beta}$.

To prove the claim, note that initially infected nodes fail to infect $a$ but succeed in infecting $b$. Let $s$ denote the size of the minimum dominating set, and $c_k\le\binom nk$ denote the number of dominating sets of size $k$. Ignoring the normalizing constant of $P[\BM y_0]$, we have:
\AL{P_{\BM\beta}[\BM y_1]\propto\sum_{k=s}^nc_k\cdot\frac1k\Big(\frac1{2n4^n}\Big)^{\!k}(1-\beta^\SIRI)^k(1-(1-\beta^\SIRI)^k)(\beta^\SIRR)^{k+1}.}
Thus, $\HAT\beta^\SIRR=1$ is the maximizer for $P_{\BM\beta}[\BM y_1]$. Plugging this into $P_{\BM\beta}[\BM y_1]$ gives
\AL{P_{\BM\beta}[\BM y_1]\propto\sum_{k=s}^nc_k\cdot\frac1k\Big(\frac1{2n4^n}\Big)^{\!k}(1-\beta^\SIRI)^k(1-(1-\beta^\SIRI)^k).}
To simplify notation, we change the variable to $\alpha:=1-\beta^\SIRI$. Then, $P_{\BM\beta}[\BM y_1]\propto p(\alpha)$ where
\AL{p(\alpha):=\sum_{k=s}^nc_k\cdot\frac1k\Big(\frac1{2n4^n}\Big)^{\!k}\alpha^k(1-\alpha^k).}
By calculus, its first order derivative is
\AL{p'\!(\alpha)=\sum_{k=s}^nc_k\Big(\frac1{2n4^n}\Big)^{\!k}\alpha^{k-1}(1-2\alpha^k),}
and then its second order derivative is
\AL{p''\!(\alpha)=\sum_{k=s}^nc_k\Big(\frac1{2n4^n}\Big)^{\!k}\alpha^{k-2}(k-1-2(2k-1)\alpha^k).}
Adding a node to a dominating set always yields a dominating set, so we have $c_k\ge1$ for all $s\le k\le n$. Since $s\le n-1$, then for each $0\le\alpha\le1/\sqrt[s]2$,
\AL{
p'\!(\alpha)
\ge{}&\Big(\frac1{2n4^n}\Big)^{\!s}\alpha^{s-1}\Big(1-2\Big(\frac1{\sqrt[s]2}\Big)^{\!s}\Big)\nonumber\\&+\sum_{k=s+1}^n\Big(\frac1{2n4^n}\Big)^{\!k}\alpha^{k-1}\Big(1-2\Big(\frac1{\sqrt[s]2}\Big)^{\!k}\Big)
\\={}&0+\!\!\!\sum_{k=s+1}^n\!\!\!\Big(\frac1{2n4^n}\Big)^{\!k}\alpha^{k-1}\Big(1-\frac2{2^{k/s}}\Big)
>0
.\label{eq:p1a-positive}}
Let $\kappa_\alpha$ denote the minimum integer $k$ such that $k-1-2(2k-1)\alpha^k>0$. For each $1/\sqrt[s]2\le\alpha\le1$, note that $\kappa_\alpha\ge s+1$ and $s-1-2(2s-1)(1/\sqrt[s]2)^s=-s$, so we have:
\AL{
p''\!(\alpha)={}&\sum_{k=s}^{\kappa_\alpha-1}c_k\Big(\frac1{2n4^n}\Big)^{\!k}\alpha^{k-2}(k-1-2(2k-1)\alpha^k)\nonumber\\&+\sum_{k=\kappa_\alpha}^nc_k\Big(\frac1{2n4^n}\Big)^{\!k}\alpha^{k-2}(k-1-2(2k-1)\alpha^k)
\\\le{}&\Big(\frac1{2n4^n}\Big)^{\!s}\alpha^{s-2}\Big(s-1-2(2s-1)\Big(\frac1{\sqrt[s]2}\Big)^{\!s}\Big)\nonumber\\&+\sum_{k=\kappa_\alpha}^n\binom nk\Big(\frac1{2n4^n}\Big)^{\!s+1}\alpha^{k-2}n
\\={}&\alpha^{s-2}\Big(\frac1{2n4^n}\Big)^{\!s}\bigg({-s+\frac1{2\cdot4^{n}}\!\!\sum_{k=\kappa_\alpha}^n\binom nk\alpha^{k-s}}\bigg)
\\\le{}&\alpha^{s-2}\Big(\frac1{2n4^n}\Big)^{\!s}\bigg({-s+\frac1{2\cdot4^{n}}\!\!\sum_{k=\kappa_\alpha}^n\binom nk\cdot1^{k-s}}\bigg)
\\\le{}&\alpha^{s-2}\Big(\frac1{2n4^n}\Big)^{\!s}\Big({-s+\frac1{2\cdot4^{n}}\cdot2^n}\Big)
\\\le{}&\alpha^{s-2}\Big(\frac1{2n4^n}\Big)^{\!s}\Big({-1+\frac14}\Big)<0
.}
This implies $p'\!(\alpha)$ is strictly decreasing over $1/\sqrt[s]2\le\alpha\le1$. Combining with Eq.~\eqref{eq:p1a-positive}, we know that $p(\alpha)$ is strictly unimodal over $0\le\alpha\le1$. Furthermore, since $c_k\ge1$ for all $s\le k\le n$, then:
\AL{p'\!(1)=-\sum_{k=s}^nc_k\Big(\frac1{2n4^n}\Big)^{\!k}\le-\sum_{k=s}^n\Big(\frac1{2n4^n}\Big)^{\!k}<0.}
Hence, the minimizer $\beta^\SIRI$ is the unique solution to $p'\!(1-\beta^\SIRI)=0$ over $0<\beta^\SIRI<1-1/\sqrt[s]2$.

Next, we give tighter bounds for $\beta^\SIRI$ to prove the claim Eq.~\eqref{eq:diffus-par-claim}. Let
\AL{\alpha_+:=\frac1{\sqrt[s]2}\bigg(1+\frac{1-\frac1{\sqrt[s]2}}{8(s+2)n4^{n}\binom ns}\bigg)>\frac1{\sqrt[s]2}.}
Note that $1-2\alpha_+^k>0$ for all $k\ge s+1$, because:
\AL{
\alpha_+&<\frac1{\sqrt[s]2}\Big(1+\frac{1-0}{8(s+2)(s+1)4^{0}}\Big)
\\&<\frac1{\sqrt[s]2}\Big(1+\frac{\log2}{s(s+1)}\Big)\le\frac1{\sqrt[s]2}\exp\!\Big(\frac{\log2}{s(s+1)}\Big)=\frac1{\sqrt[s+1]2}
.}
Since $1\le s\le n-1$, and $1\le c_k\le\binom nk$ for $s\le k\le n$, then by \LEMref{bern-ub},
\AL{
p'\!(\alpha_+)\ge{}&{-\binom ns\Big(\frac1{2n4^n}\Big)^{\!s}\alpha_+^{s-1}(2\alpha_+^s-1)}\nonumber\\&+\Big(\frac1{2n4^n}\Big)^{\!s+1}\alpha_+^{s}(1-2\alpha_+^{s+1})
\\={}&\frac{\alpha_+^{s-1}}{(2n4^n)^s}\bigg({-\binom ns}(2\alpha_+^s-1)+\frac{\alpha_+(1-2\alpha_+^{s+1})}{2n4^n}\bigg)
\\={}&\frac{\alpha_+^{s-1}}{(2n4^n)^s}\bigg({-\binom ns}\bigg(\bigg(1+\frac{1-\frac1{\sqrt[s]2}}{8(s+2)n4^{n}\binom ns}\bigg)^{\!s}-1\bigg)\nonumber\\&+\frac{\alpha_+\Big(1-\frac1{\sqrt[s]2}\Big(1+\frac{1-\frac1{\sqrt[s]2}}{8(s+2)n4^{n}\binom ns}\Big)^{\!s+1}\Big)\Big)}{2n4^n}
\\\ge{}&\frac{\alpha_+^{s-1}}{(2n4^n)^s}\Bigg({-\binom ns}\bigg(\bigg(1+\frac{(s+1)\big(1-\frac1{\sqrt[s]2}\big)}{8(s+2)n4^{n}\binom ns}\bigg)-1\bigg)\nonumber\\&+\frac{\frac1{\sqrt[s]2}\Big(1-\frac1{\sqrt[s]2}\Big(1+\frac{(s+2)\big(1-\frac1{\sqrt[s]2}\big)}{8(s+2)n4^{n}\binom ns}\Big)\Big)}{2n4^n}\Bigg)
\\={}&\frac{\frac1{\sqrt[s]2}\big(1-\frac1{\sqrt[s]2}\big)\alpha_+^{s-1}}{(2n4^n)^{s+1}}\bigg(1-\frac{\sqrt[s]2(s+1)}{4(s+2)}-\frac{1}{8\sqrt[s]2n4^{n}\binom ns}\bigg)
\\>{}&\frac{\frac1{\sqrt[s]2}\big(1-\frac1{\sqrt[s]2}\big)\alpha_+^{s-1}}{(2n4^n)^{s+1}}\Big(1-\frac12-\frac18\Big)>0
.}
This implies the maximizer $\beta^\SIRI<1-\alpha_+$. Hence,
\AL{\HAT\beta^\SIRI&<(1+\epsilon)\beta^\SIRI<(1+\epsilon)(1-\alpha_+)
\\&\le\Big(1+\frac1{16(n+1)n8^n-1}\Big)\bigg(1-\frac1{\sqrt[s]2}\bigg(1+\frac{1-\frac1{\sqrt[s]2}}{8(s+2)n4^{n}\binom ns}\bigg)\bigg)
\\&\le\Big(1+\frac1{8\sqrt[s]2(s+2)n4^n\binom ns-1}\Big)\bigg(1-\frac1{\sqrt[s]2}\bigg(1+\frac{1-\frac1{\sqrt[s]2}}{8(s+2)n4^{n}\binom ns}\bigg)\bigg)
\\&=1-\frac1{\sqrt[s]2}
.\label{eq:mle-ub}}
For the lower bound, let
\AL{\alpha_-:=\frac1{\sqrt[s]2}\Big(1+\frac1{n2^n}\Big)^{\!\frac1s}>\frac1{\sqrt[s]2}.}
Since $s\le n-1$, we have:
\AL{\alpha_-\le{}&\frac1{\sqrt[s]2}\Big(1+\frac1{(s+1)2^{s+1}}\Big)^{\!\frac1s}
\\<{}&\frac1{\sqrt[s]2}\Big(1+\frac{\log2}{s+1}\Big)^{\!\frac1s}<\frac1{\sqrt[s]2}\big(\RM e^{\frac{\log2}{s+1}}\big)^{\!\frac1s}=\frac1{\sqrt[s+1]2}
.}
Besides that, note that
\AL{\alpha_-^s(2\alpha_-^s-1)=\frac1{2n2^n}\Big(1+\frac1{n2^n}\Big)>\frac1{2n2^n}.}
Since $1\le s\le n-1$, and $1\le c_k\le\binom nk$ for $s\le k\le n$, then:
\AL{p'\!(\alpha_-)\le{}&{-\Big(\frac1{2n4^n}\Big)^{\!s}}\alpha_-^{s-1}(2\alpha_-^s-1)\nonumber\\&
+\sum_{k=s+1}^n\binom nk\Big(\frac1{2n4^n}\Big)^{\!k}\alpha_-^{k-1}(1-2\alpha_-^k)
\\\le{}&{-\Big(\frac1{2n4^n}\Big)^{\!s}}\alpha_-^{s-1}(2\alpha_-^s-1)+\sum_{k=s+1}^n\binom nk\Big(\frac1{2n4^n}\Big)^{\!s+1}\alpha_-^{k-1}
\\={}&\frac{\alpha_-^{-1}}{(2n4^n)^s}\bigg({-\alpha_-^s}(2\alpha_-^s-1)+\frac1{2n4^n}\sum_{k=s+1}^n\binom nk\alpha_-^k\bigg)
\\\le{}&\frac{\alpha_-^{-1}}{(2n4^n)^s}\bigg({-\alpha_-^s}(2\alpha_-^s-1)+\frac1{2n4^n}\sum_{k=0}^n\binom nk\alpha_-^k\bigg)
\\={}&\frac{\alpha_-^{-1}}{(2n4^n)^s}\Big({-\alpha_-^s}(2\alpha_-^s-1)+\frac1{2n4^n}(1+\alpha_-)^n\Big)
\\<{}&\frac{\alpha_-^{-1}}{(2n4^n)^s}\Big({-\alpha_-^s}(2\alpha_-^s-1)+\frac1{2n4^n}2^n\Big)
\\<{}&\frac{\alpha_-^{-1}}{(2n4^n)^s}\Big({-\frac1{2n2^n}}+\frac1{2n4^n}2^n\Big)=0
.}
This implies the maximizer $\beta^\SIRI>1-\alpha_-$. Hence,
\AL{\HAT\beta^\SIRI&>(1-\epsilon)\beta^\SIRI>(1-\epsilon)(1-\alpha_-)
\\&\ge\left(1-\frac{\frac1{\sqrt[n]2}-\frac1{\sqrt[n-1]2}\big(1+\frac1{n2^n}\big)^{\!\frac1{n-1}}}{1-\frac1{\sqrt[n-1]2}\big(1+\frac1{n2^n}\big)^{\!\frac1{n-1}}}\right)\bigg(1-\frac1{\sqrt[s]2}\Big(1+\frac1{n2^n}\Big)^{\!\frac1s}\bigg)
\\&\ge\left(1-\frac{\frac1{\sqrt[s+1]2}-\frac1{\sqrt[s]2}\big(1+\frac1{n2^n}\big)^{\!\frac1{s}}}{1-\frac1{\sqrt[s]2}\big(1+\frac1{n2^n}\big)^{\!\frac1{s}}}\right)\bigg(1-\frac1{\sqrt[s]2}\Big(1+\frac1{n2^n}\Big)^{\!\frac1s}\bigg)
\\&=1-\frac1{\sqrt[s+1]2}
.\label{eq:mle-lb}}
Combining Eq.~\eqref{eq:mle-ub} and Eq.~\eqref{eq:mle-lb} yields our claim Eq.~\eqref{eq:diffus-par-claim}.

The numbers involved can be stored in $\Poly(n)$ bits and be computed using high-precision arithmetics within $\Poly(n)$ time. Therefore, this gives a polynomial-time reduction from the MDS problem to \PRBref{diffus-par-mle}, so the NP-hardness of the MDS problem \cite{karp-npc} implies that \PRBref{diffus-par-mle} is NP-hard.
\end{proof}

\ASSEC{Proof of \THMref{diffus-mle-sens}}{app:pf-3}
Before proving \THMref{diffus-mle-sens}, we give an auxiliary lemma.

\begin{LEM}\label{LEM:sir-hist-prob}
Under the SIR model, for each possible history $\BM Y\in\OP{supp}(P)$, there exists a number $\omega_{\BM Y}>0$ independent of $\BM\beta$ such that
\AL{
P_{\BM\beta}[\BM Y]=\omega_{\BM Y}(\beta^\SIRI)^{n^{\SIRI\SIRR}(\BM y_T)-n^{\SIRI\SIRR}(\BM y_0)}(\beta^\SIRR)^{n^\SIRR(\BM y_T)-n^\SIRR(\BM y_0)}(1+\OP O(\|\BM\beta\|)).}
\end{LEM}

\begin{proof}
Fix the history $\BM Y$, and we will omit $\BM Y$ in some notations.

Under the SIR model, we have the factorization:
\AL{
P_{\BM\beta}[\BM Y]&=P[\BM y_0]\prod_{t=1}^T\prod_{u\in\CAL V}P_{\BM\beta}[y_{t,u}\mid\BM y_{t-1}]\\&=P[\BM y_0]\prod_{u\in\CAL V}\bigg(\prod_{t=1}^TP_{\BM\beta}[y_{t,u}\mid\BM y_{t-1}]\bigg).
}
Let $\CAL U_{x_0}^{x_1}$ denote the set of nodes with initial state $x_0$ and final state $x_1$:
\AL{\CAL U_{x_0}^{x_1}:=\{u\in\CAL V : y_{0,u}=x_0,\,y_{T,u}=x_1\}.\label{eq:diffus-prob-decomp}}
Then, the node set $\CAL V$ can be decomposed disjointly into
\AL{\CAL V=\CAL U_{\SIRS}^{\SIRS}\cup\CAL U_\SIRS^\SIRI\cup\CAL U_\SIRS^\SIRR\cup\CAL U_\SIRI^\SIRI\cup\CAL U_\SIRI^\SIRR\cup\CAL U_\SIRR^\SIRR.}
Besides that, for each node $u\in\CAL U_\SIRS^\SIRI\cup\CAL U_\SIRS^\SIRR$, let $\CAL I_u$ denote the set of neighbors that may have infected $u$:
\AL{\CAL I_u:=\{v\in\CAL N_u:y_{h_u^\SIRI-1,u}=\SIRI\}\ne\varnothing.}

Now we calculate $\prod_{t=1}^{T}P_{\BM\beta}[y_{t,u}\mid\BM y_{t-1}]$ according to the decomposition Eq.~\eqref{eq:diffus-prob-decomp}. For each node $u\in\CAL U_\SIRS^\SIRS$, it is never infected, so:
\AL{
\prod_{t=1}^{T}P_{\BM\beta}[y_{t,u}\mid\BM y_{t-1}]&=\prod_{t=1}^{T}\prod_{v\in\CAL N_u\wedge y_{t-1,v}=\SIRI}\!\!\!\!\!\!\!\!(1-\beta^\SIRI)
\\&=\prod_{t=1}^{T}(1+\OP O(\|\BM\beta\|))=1+\OP O(\|\BM\beta\|)
.}
For each node $u\in\CAL U_\SIRS^\SIRI$, it is infected but never recovered, so:
\AL{
&\prod_{t=1}^{T}P_{\BM\beta}[y_{t,u}\mid\BM y_{t-1}]\nonumber
\\={}&\bigg(\prod_{t=1}^{h_u^\SIRI-1}\!\!\prod_{v\in\CAL N_u\wedge y_{t-1,v}=\SIRI}\!\!\!\!\!\!\!\!\!\!(1-\beta^\SIRI)\bigg)\bigg(1-\!\!\!\!\!\!\!\!\!\prod_{v\in\CAL N_u\wedge y_{h_u^\SIRI-1,v}=\SIRI}\!\!\!\!\!\!\!\!\!\!\!\!(1-\beta^\SIRI)\bigg)\bigg(\!\prod_{t=h_u^\SIRI}^T(1-\beta^\SIRR)\bigg)
\\={}&\bigg(\prod_{t=1}^{h_u^\SIRI-1}\!\!\prod_{v\in\CAL N_u\wedge y_{t-1,v}=\SIRI}\!\!\!\!\!\!\!\!\!\!(1-\beta^\SIRI)\bigg)\big(1-(1-\beta^\SIRI)^{|\CAL I_u|}\big)\bigg(\!\prod_{t=h_u^\SIRI}^T(1-\beta^\SIRR)\bigg)
\\={}&(1+\OP O(\|\BM\beta\|))\cdot(|\CAL I_u|\beta^\SIRI(1+\OP O(\|\BM\beta\|)))\cdot(1+\OP O(\|\BM\beta\|))
\\={}&|\CAL I_u|\beta^\SIRI(1+\OP O(\|\BM\beta\|))
.}
For each node $u\in\CAL U_\SIRS^\SIRR$, it is infected and recovered, so:
\AL{
&\prod_{t=1}^{T}P_{\BM\beta}[y_{t,u}\mid\BM y_{t-1}]\nonumber
\\={}&\bigg(\prod_{t=1}^{h_u^\SIRI-1}\!\!\prod_{v\in\CAL N_u\wedge y_{t-1,v}=\SIRI}\!\!\!\!\!\!\!\!\!\!(1-\beta^\SIRI)\bigg)\bigg(1-\!\!\!\!\!\!\!\!\!\prod_{v\in\CAL N_u\wedge y_{h_u^\SIRI-1,v}=\SIRI}\!\!\!\!\!\!\!\!\!\!\!\!(1-\beta^\SIRI)\bigg)\bigg(\!\prod_{t=h_u^\SIRI}^{h_u^\SIRR}(1-\beta^\SIRR)\bigg)\beta^\SIRR
\\={}&(1+\OP O(\|\BM\beta\|))\cdot(|\CAL I_u|\beta^\SIRI(1+\OP O(\|\BM\beta\|)))\cdot(1+\OP O(\|\BM\beta\|))\cdot\beta^\SIRR
\\={}&|\CAL I_u|\beta^\SIRI\beta^\SIRR(1+\OP O(\|\BM\beta\|))
.}
For each node $u\in\CAL U_\SIRI^\SIRI$, it is never recovered, so:
\AL{
\prod_{t=1}^{T}P_{\BM\beta}[y_{t,u}\mid\BM y_{t-1}]=\prod_{t=1}^T(1-\beta^\SIRR)=(1+\OP O(\|\BM\beta\|))
.}
For each node $u\in\CAL U_\SIRI^\SIRR$, it is eventually recovered, so:
\AL{
&\prod_{t=1}^{T}P_{\BM\beta}[y_{t,u}\mid\BM y_{t-1}]=\bigg(\!\prod_{t=1}^{h_u^\SIRR}(1-\beta^\SIRR)\bigg)\beta^\SIRR
\\={}&(1+\OP O(\|\BM\beta\|))\cdot\beta^\SIRR=\beta^\SIRR(1+\OP O(\|\BM\beta\|))
.}
For each node $u\in\CAL U_\SIRR^\SIRR$, its state does not change, so:
\AL{
\prod_{t=1}^{T}P_{\BM\beta}[y_{t,u}\mid\BM y_{t-1}]=\prod_{t=1}^{T}1=1
.}
Finally, note that
\AL{|\CAL U_\SIRS^\SIRI\cup\CAL U_\SIRS^\SIRR|&=n^{\SIRI\SIRR}(\BM y_T)-n^{\SIRI\SIRR}(\BM y_0),\\
|\CAL U_\SIRS^\SIRR\cup\CAL U_\SIRI^\SIRR|&=n^{\SIRR}(\BM y_T)-n^{\SIRR}(\BM y_0).}
Therefore,
\AL{
&P_{\BM\beta}[\BM Y]=P[\BM y_0]\prod_{u\in\CAL V}\bigg(\prod_{t=1}^TP_{\BM\beta}[y_{t,u}\mid\BM y_{t-1}]\bigg)
\\={}&P[\BM y_0]\prod_{u\in\CAL U_{\SIRS}^{\SIRS}\cup\CAL U_\SIRS^\SIRI\cup\CAL U_\SIRS^\SIRR\cup\CAL U_\SIRI^\SIRI\cup\CAL U_\SIRI^\SIRR\cup\CAL U_\SIRR^\SIRR}\bigg(\prod_{t=1}^TP_{\BM\beta}[y_{t,u}\mid\BM y_{t-1}]\bigg)
\\={}&P[\BM y_0]\bigg(\prod_{u\in\CAL U_{\SIRS}^{\SIRS}}\prod_{t=1}^TP_{\BM\beta}[y_{t,u}\mid\BM y_{t-1}]\bigg)\bigg(\prod_{u\in\CAL U_{\SIRS}^{\SIRI}}\prod_{t=1}^TP_{\BM\beta}[y_{t,u}\mid\BM y_{t-1}]\bigg)\nonumber\\&\qquad\;\bigg(\prod_{u\in\CAL U_{\SIRS}^{\SIRR}}\prod_{t=1}^TP_{\BM\beta}[y_{t,u}\mid\BM y_{t-1}]\bigg)\bigg(\prod_{u\in\CAL U_{\SIRI}^{\SIRI}}\prod_{t=1}^TP_{\BM\beta}[y_{t,u}\mid\BM y_{t-1}]\bigg)\nonumber\\&\qquad\;\bigg(\prod_{u\in\CAL U_{\SIRI}^{\SIRR}}\prod_{t=1}^TP_{\BM\beta}[y_{t,u}\mid\BM y_{t-1}]\bigg)\bigg(\prod_{u\in\CAL U_{\SIRR}^{\SIRR}}\prod_{t=1}^TP_{\BM\beta}[y_{t,u}\mid\BM y_{t-1}]\bigg)
\\={}&P[\BM y_0]\bigg(\prod_{u\in\CAL U_{\SIRS}^{\SIRS}}(1+\OP O(\|\BM\beta\|))\bigg)\bigg(\prod_{u\in\CAL U_{\SIRS}^{\SIRI}}(|\CAL I_u|\beta^\SIRI(1+\OP O(\|\BM\beta\|)))\bigg)\nonumber\\&\qquad\;\bigg(\prod_{u\in\CAL U_{\SIRS}^{\SIRR}}(|\CAL I_u|\beta^\SIRI\beta^\SIRR(1+\OP O(\|\BM\beta\|)))\bigg)\bigg(\prod_{u\in\CAL U_{\SIRI}^{\SIRI}}(1+\OP O(\|\BM\beta\|))\bigg)\nonumber\\&\qquad\;\bigg(\prod_{u\in\CAL U_{\SIRI}^{\SIRR}}(\beta^\SIRR(1+\OP O(\|\BM\beta\|)))\bigg)\bigg(\prod_{u\in\CAL U_{\SIRR}^{\SIRR}}\prod_{t=1}^T1\bigg)
\\={}&\bigg(P[\BM y_0]\!\!\!\!\prod_{u\in\CAL U_\SIRS^\SIRI\cup\CAL U_\SIRS^\SIRR}\!\!\!\!|\CAL I_u|\bigg)(\beta^\SIRI)^{|\CAL U_\SIRS^\SIRI\cup\CAL U_\SIRS^\SIRR|}(\beta^\SIRR)^{|\CAL U_\SIRS^\SIRR\cup\CAL U_\SIRI^\SIRR|}(1+\OP O(\|\BM\beta\|))
\\={}&\omega_{\BM Y}(\beta^\SIRI)^{n^{\SIRI\SIRR}(\BM y_T)-n^{\SIRI\SIRR}(\BM y_0)}(\beta^\SIRR)^{n^\SIRR(\BM y_T)-n^\SIRR(\BM y_0)}(1+\OP O(\|\BM\beta\|))
,}
where $\omega_{\BM Y}:=P[\BM y_0]\prod_{u\in\CAL U_\SIRS^\SIRI\cup\CAL U_\SIRS^\SIRR}|\CAL I_u| > 0$ because $\CAL I_u\ne\varnothing$.
\end{proof}

Now we are ready to prove \THMref{diffus-mle-sens}.

\begin{proof}[Proof of \THMref{diffus-mle-sens}]
By \LEMref{sir-hist-prob},
\AL{P_{\BM\beta}[\BM Y]=\omega_{\BM Y}(\beta^\SIRI)^{n^{\SIRI\SIRR}(\BM y_T)-n^{\SIRI\SIRR}(\BM y_0)}(\beta^\SIRR)^{n^\SIRR(\BM y_T)-n^\SIRR(\BM y_0)}(1+\OP O(\|\BM\beta\|)).}
Thus,
\AL{\log P_{\BM\beta}[\BM Y]={}&\log\omega_{\BM Y}+(n^{\SIRI\SIRR}(\BM y_T)-n^{\SIRI\SIRR}(\BM y_0))\log\beta^\SIRI\\&+(n^\SIRR(\BM y_T)-n^\SIRR(\BM y_0))\log\beta^\SIRR+\log(1+\OP O(\|\BM\beta\|))
\\={}&\log\omega_{\BM Y}+(n^{\SIRI\SIRR}(\BM y_T)-n^{\SIRI\SIRR}(\BM y_0))\log\beta^\SIRI\\&+(n^\SIRR(\BM y_T)-n^\SIRR(\BM y_0))\log\beta^\SIRR+\OP O(\|\BM\beta\|)
.}
Note that $P_{\BM\beta}[\BM Y]$ is a polynomial of $\BM\beta$, so it is differentiable w.r.t.\ $\BM\beta$. Hence,
\AL{
\frac{\partial}{\partial\beta^\SIRI}\log P_{\BM\beta}[\BM Y]&=\frac{n^{\SIRI\SIRR}(\BM y_T)-n^{\SIRI\SIRR}(\BM y_0)}{\beta^\SIRI}+\OP O(1),\\
\frac{\partial}{\partial\beta^\SIRR}\log P_{\BM\beta}[\BM Y]&=\frac{n^\SIRR(\BM y_T)-n^\SIRR(\BM y_0)}{\beta^\SIRR}+\OP O(1).}
It follows that
\AL{\frac{\partial}{\partial\beta^\SIRI}P_{\BM\beta}[\BM Y]&=\Big(\frac{\partial}{\partial\beta^\SIRI}\log P_{\BM\beta}[\BM Y]\Big)P_{\BM\beta}[\BM Y]\\
&=\Big(\frac{n^{\SIRI\SIRR}(\BM y_T)-n^{\SIRI\SIRR}(\BM y_0)}{\beta^\SIRI}+\OP O(1)\Big)P_{\BM\beta}[\BM Y],\\
\frac{\partial}{\partial\beta^\SIRR}P_{\BM\beta}[\BM Y]&=\Big(\frac{\partial}{\partial\beta^\SIRR}\log P_{\BM\beta}[\BM Y]\Big)P_{\BM\beta}[\BM Y]\\
&=\Big(\frac{n^{\SIRR}(\BM y_T)-n^{\SIRR}(\BM y_0)}{\beta^\SIRR}+\OP O(1)\Big)P_{\BM\beta}[\BM Y],}
Therefore, if $n^{\SIRI\SIRR}(\BM y_T)>n^{\SIRI\SIRR}(\BM y_0)$, then:
\AL{\frac{\partial}{\partial\beta^\SIRI}P_{\BM\beta}[\BM Y]=\Theta\Big(\frac1{\beta^\SIRI}\Big)P_{\BM\beta}[\BM Y];}
if $n^\SIRR(\BM y_T)>n^\SIRR(\BM y_0)$, then:
\AL{\frac{\partial}{\partial\beta^\SIRR}P_{\BM\beta}[\BM Y]=\Theta\Big(\frac1{\beta^\SIRR}\Big)P_{\BM\beta}[\BM Y].}
\end{proof}

\ASSEC{Proof of \THMref{hit-stable}}{app:pf-4}


\begin{proof}
Since $n^\SIRI(\BM y_0)$ and $n^\SIRR(\BM y_0)$ are fixed a.s., we write $n^{\SIRI\SIRR}_0:=n^{\SIRI\SIRR}(\BM y_0)$ and $n^\SIRR_0:=n^\SIRR(\BM y_0)$ are constants. Fix a snapshot $\BM y_T$, then $n^{\SIRI\SIRR}_T:=n^{\SIRI\SIRR}(\BM y_T)$ and $n^{\SIRR}_T:=n^{\SIRR}(\BM y_T)$ are also fixed. Then by \LEMref{sir-hist-prob}, for any history $\BM Y\in\OP{supp}(P\mathbin|\BM y_T)$,
\AL{P_{\BM\beta}[\BM Y]=\omega_{\BM Y}(\beta^\SIRI)^{n^{\SIRI\SIRR}_T-n^{\SIRI\SIRR}_0}(\beta^\SIRR)^{n^{\SIRR}_T-n^{\SIRR}_0}(1+\OP O(\|\BM\beta\|)).}
Thus, the probability of the snapshot $\BM y_T$ is
\AL{P_{\BM\beta}[\BM y_T]&=\!\!\!\!\!\!\!\!\sum_{\BM Y\in\OP{supp}(P\mid\BM y_T)}\!\!\!\!\!\!\!\!P_{\BM\beta}[\BM Y]
\\&=\!\!\!\!\!\!\!\!\sum_{\BM Y\in\OP{supp}(P\mid\BM y_T)}\!\!\!\!\!\!\!\!\omega_{\BM Y}(\beta^\SIRI)^{n^{\SIRI\SIRR}_T-n^{\SIRI\SIRR}_0}(\beta^\SIRR)^{n^{\SIRR}_T-n^{\SIRR}_0}(1+\OP O(\|\BM\beta\|))
.}
Fix a node $u\in\CAL V$. Categorize the possible histories according to the hitting times $0\le t\le T+1$ of the node $u$:
\AL{
\CAL Y^\SIRI_t&:=\{\BM Y\in\OP{supp}(P\mathbin|\BM y_T):h_u^\SIRI(\BM Y)=t\},\\
\CAL Y^\SIRR_t&:=\{\BM Y\in\OP{supp}(P\mathbin|\BM y_T):h_u^\SIRR(\BM Y)=t\}.}
Let
\AL{
\omega_t^\SIRI:=\sum_{\BM Y\in\CAL Y^\SIRI_t}\omega_{\BM Y},\quad
\omega_t^\SIRR:=\sum_{\BM Y\in\CAL Y^\SIRR_t}\omega_{\BM Y},\quad\omega:=\!\!\!\!\!\!\!\!\sum_{\BM Y\in\OP{supp}(P\mid\BM y_T)}\!\!\!\!\!\!\!\!\omega_{\BM Y}.
}
Then by cancellation, the expected hitting time to state $\SIRI$ is
\AL{
&\Exp_{\BM Y\sim P_{\BM\beta}\mid\BM y_T}[h_u^\SIRI(\BM Y)]
\\={}&\sum_{t=0}^{T+1}t\sum_{\BM Y\in\CAL Y^\SIRI_t}P_{\BM\beta}[\BM Y\mathbin|\BM y_T]
\\={}&\frac{\sum\limits_{t=0}^{T+1}t\sum\limits_{\BM Y\in\CAL Y^\SIRI_t}P_{\BM\beta}[\BM Y]}{P_{\BM\beta}[\BM y_T]}
\\={}&\frac{\sum\limits_{t=0}^{T+1}t\sum\limits_{\BM Y\in\CAL Y^\SIRI_t}\omega_{\BM Y}(\beta^\SIRI)^{n^{\SIRI\SIRR}_T-n^{\SIRI\SIRR}_0}(\beta^\SIRR)^{n^{\SIRR}_T-n^{\SIRR}_0}(1+\OP O(\|\BM\beta\|))}{\sum\limits_{\BM Y\in\OP{supp}(P\mid\BM y_T)}\!\!\!\!\!\!\!\!\omega_{\BM Y}(\beta^\SIRI)^{n^{\SIRI\SIRR}_T-n^{\SIRI\SIRR}_0}(\beta^\SIRR)^{n^{\SIRR}_T-n^{\SIRR}_0}(1+\OP O(\|\BM\beta\|))}
\\={}&\frac{\sum\limits_{t=0}^{T+1}t\sum\limits_{\BM Y\in\CAL Y^\SIRI_t}\omega_{\BM Y}(1+\OP O(\|\BM\beta\|))}{\sum\limits_{\BM Y\in\OP{supp}(P\mid\BM y_T)}\!\!\!\!\!\!\!\!\omega_{\BM Y}(1+\OP O(\|\BM\beta\|))}
\\={}&\frac{\sum\limits_{t=0}^{T+1}t\omega_t^\SIRI+\OP O(\|\BM\beta\|)}{\omega+\OP O(\|\BM\beta\|)}
\\={}&\bigg(\sum_{t=0}^{T+1}t\omega_t^\SIRI+\OP O(\|\BM\beta\|)\bigg)\Big(\frac1\omega+\OP O(\|\BM\beta\|)\Big)
\\={}&\frac1\omega\sum_{t=0}^{T+1}t\omega_t^\SIRI+\OP O(\|\BM\beta\|)
.}
Similarly, the expected hitting time to state $\SIRR$ is
\AL{\Exp_{\BM Y\sim P_{\BM\beta}\mid\BM y_T}[h_u^\SIRR(\BM Y)]=\frac1\omega\sum_{t=0}^{T+1}t\omega_t^\SIRR+\OP O(\|\BM\beta\|).}
Therefore,
\AL{\nabla_{\!\BM\beta}\!\Exp_{\BM Y\sim P_{\BM\beta}\mid\BM y_T}\!\!\!\!\!\![h_u^\SIRI(\BM Y)]&=\nabla_{\!\BM\beta}\bigg(\frac1\omega\sum_{t=0}^{T+1}t\omega_t^\SIRI+\OP O(\|\BM\beta\|)\bigg)=\OP O(1),\\
\nabla_{\!\BM\beta}\!\Exp_{\BM Y\sim P_{\BM\beta}\mid\BM y_T}\!\!\!\!\!\![h_u^\SIRR(\BM Y)]&=\nabla_{\!\BM\beta}\bigg(\frac1\omega\sum_{t=0}^{T+1}t\omega_t^\SIRR+\OP O(\|\BM\beta\|)\bigg)=\OP O(1).}
\end{proof}

\ASSEC{Proof of \THMref{q-equiv-obj}}{app:pf-5}
\begin{proof}
The sufficient expressiveness of $Q_{\BM\theta}$ implies that there exists a parameter sequence $\{\BM\theta_k\}_{k\ge1}$ such that
\AL{\lim_{k\to+\infty}Q_{\BM\theta_k}(\BM y_T)[\BM Y]=P_{\HAT{\BM\beta}}[\BM Y\mid\BM y_T],\qquad\forall\BM y_T.}
Since $Q_{\BM\theta}(\BM y_T)$ and $P_{\HAT{\BM\beta}}\mathbin|\BM y_T$ share a common finite support, then $\max_{\BM Y\in\OP{supp}(P_{\HAT{\BM\beta}})}\big|{\log Q_{\BM\theta}(\BM y_T)[\BM Y]-\log P_{\HAT{\BM\beta}}[\BM Y\mid\BM y_T]}\big|<\infty$. It follows from the dominated convergence theorem that
\AL{0&\le\min_{\BM\theta}\Exp_{\BM Y\sim P_{\HAT{\BM\beta}}}[(\log Q_{\BM\theta}(\BM y_T)[\BM Y]-\log P_{\HAT{\BM\beta}}[\BM Y\mid\BM y_T])^2]
\\&\le\lim_{k\to+\infty}\Exp_{\BM Y\sim P_{\HAT{\BM\beta}}}[(\log Q_{\BM\theta_k}(\BM y_T)[\BM Y]-\log P_{\HAT{\BM\beta}}[\BM Y\mid\BM y_T])^2]
\\&=\Exp_{\BM Y\sim P_{\HAT{\BM\beta}}}\Big[\lim_{k\to+\infty}(\log Q_{\BM\theta_k}(\BM y_T)[\BM Y]-\log P_{\HAT{\BM\beta}}[\BM Y\mid\BM y_T])^2\Big]
\\&=\Exp_{\BM Y\sim P_{\HAT{\BM\beta}}}\Big[\Big(\log\lim_{k\to+\infty}Q_{\BM\theta_k}(\BM y_T)[\BM Y]-\log P_{\HAT{\BM\beta}}[\BM Y\mid\BM y_T]\Big)^{\!2}\Big]
\\&=\Exp_{\BM Y\sim P_{\HAT{\BM\beta}}}[(\log P_{\HAT{\BM\beta}}[\BM Y\mid\BM y_T]-\log P_{\HAT{\BM\beta}}[\BM Y\mid\BM y_T])^2]
\\&=\Exp_{\BM Y\sim P_{\HAT{\BM\beta}}}[0^2]=0
.}
This implies
\AL{
&\lim_{k\to+\infty}\Exp_{\BM Y\sim P_{\HAT{\BM\beta}}}[(\log Q_{\BM\theta_k}(\BM y_T)[\BM Y]-\log P_{\HAT{\BM\beta}}[\BM Y\mid\BM y_T])^2]\\
={}&\min_{\BM\theta}\Exp_{\BM Y\sim P_{\HAT{\BM\beta}}}[(\log Q_{\BM\theta}(\BM y_T)[\BM Y]-\log P_{\HAT{\BM\beta}}[\BM Y\mid\BM y_T])^2]=0.}
Hence, $\{\BM\theta_k\}_{k\ge1}$ is asymptotically optimal for the original objective Eq.~\eqref{eq:q-orig-obj}. Meanwhile, for any other parameter sequence $\{\TLD{\BM\theta}_k\}$ where $Q_{\TLD{\BM\theta}_k}(\BM y_T)$ do not converge to $P_{\HAT{\BM\beta}}\mathbin|\BM y_T$, then they have nonzero objectives and are thus non-optimal. Therefore, any asymptotically optimal parameter sequence $\{\BM\theta_k\}_{k\ge1}$ for the objective Eq.~\eqref{eq:q-orig-obj} must converge to $P_{\HAT{\BM\beta}}\mathbin|\BM y_T$.

Next, we will show that any asymptotically optimal parameter sequence $\{\BM\theta_k\}_{k\ge1}$ for the objective Eq.~\eqref{eq:q-equiv-obj} must also converge to $P_{\HAT{\BM\beta}}\mathbin|\BM y_T$. For any $\BM y_T$, since $\OP{supp}(Q_{\BM\theta}(\BM y_T))=\OP{supp}(P_{\HAT{\BM\beta}}\mathbin|\BM y_T)$, then:
\AL{
&\Exp_{\BM Y\sim P_{\HAT{\BM\beta}}\mid\BM y_T}\Big[\frac{Q_{\BM\theta}(\BM y_T)[\BM Y]}{P_{\HAT{\BM\beta}}[\BM Y]}\Big]
\\={}&\sum_{\BM Y\in\OP{supp}(P_{\HAT{\BM\beta}}\mid\BM y_T)}P_{\HAT{\BM\beta}}[\BM Y\mid\BM y_T]\cdot\frac{Q_{\BM\theta}(\BM y_T)[\BM Y]}{P_{\HAT{\BM\beta}}[\BM Y]}
\\={}&\frac1{P_{\HAT{\BM\beta}}[\BM y_T]}\cdot\sum_{\BM Y\in\OP{supp}(P_{\HAT{\BM\beta}}\mid\BM y_T)}Q_{\BM\theta}(\BM y_T)[\BM Y]
\\={}&\frac1{P_{\HAT{\BM\beta}}[\BM y_T]}\cdot\sum_{\BM Y\in\OP{supp}(Q_{\BM\theta}(\BM y_T))}Q_{\BM\theta}(\BM y_T)[\BM Y]
\\={}&\frac1{P_{\HAT{\BM\beta}}[\BM y_T]}\cdot1=\frac1{P_{\HAT{\BM\beta}}[\BM y_T]}.}
By Jensen's inequality,
\AL{
&\Exp_{\BM Y\sim P_{\HAT{\BM\beta}}\mid\BM y_T}\Big[\psi\Big(\frac{Q_{\BM\theta}(\BM y_T)[\BM Y]}{P_{\HAT{\BM\beta}}[\BM Y]}\Big)\Big]
\\\ge{}&\psi\Big(\Exp_{\BM Y\sim P_{\HAT{\BM\beta}}\mid\BM y_T}\Big[\frac{Q_{\BM\theta}(\BM y_T)[\BM Y]}{P_{\HAT{\BM\beta}}[\BM Y]}\Big]\Big)
\\={}&\psi\Big(\frac1{P_{\HAT{\BM\beta}}[\BM y_T]}\Big)
.}
Since convexity implies continuity, then by the law of total expectation,
\AL{
&\min_{\BM\theta}\Exp_{\BM Y\sim P_{\HAT{\BM\beta}}}\Big[\psi\Big(\frac{Q_{\BM\theta}(\BM y_T)[\BM Y]}{P_{\HAT{\BM\beta}}[\BM Y]}\Big)\Big]
\\={}&\min_{\BM\theta}\Exp_{\BM y_T\sim P_{\HAT{\BM\beta}}}\Big[\Exp_{\BM Y\sim P_{\HAT{\BM\beta}}\mid\BM y_T}\Big[\psi\Big(\frac{Q_{\BM\theta}(\BM y_T)[\BM Y]}{P_{\HAT{\BM\beta}}[\BM Y]}\Big)\Big]\Big]
\\\ge{}&\min_{\BM\theta}\Exp_{\BM y_T\sim P_{\HAT{\BM\beta}}}\Big[\psi\Big(\frac1{P_{\HAT{\BM\beta}}[\BM y_T]}\Big)\Big]
\\={}&\Exp_{\BM y_T\sim P_{\HAT{\BM\beta}}}\Big[\psi\Big(\frac1{P_{\HAT{\BM\beta}}[\BM y_T]}\Big)\Big]
\\={}&\Exp_{\BM Y\sim P_{\HAT{\BM\beta}}}\Big[\psi\Big(\frac1{P_{\HAT{\BM\beta}}[\BM y_T]}\Big)\Big]\label{eq:quo-lower-bd}
.}
Thus, for any parameter sequence $\{\BM\theta_k\}$ such that $Q_{\BM\theta_k}(\BM y_T)$ converge to $P_{\HAT{\BM\beta}}\mathbin|\BM y_T$,
\AL{
\Exp_{\BM Y\sim P_{\HAT{\BM\beta}}}\Big[\psi\Big(\frac1{P_{\HAT{\BM\beta}}[\BM y_T]}\Big)\Big]
\le&\min_{\BM\theta}\Exp_{\BM Y\sim P_{\HAT{\BM\beta}}}\Big[\psi\Big(\frac{Q_{\BM\theta}(\BM y_T)[\BM Y]}{P_{\HAT{\BM\beta}}[\BM Y]}\Big)\Big]
\\\le{}&\lim_{k\to+\infty}\Exp_{\BM Y\sim P_{\HAT{\BM\beta}}}\Big[\psi\Big(\frac{Q_{\BM\theta_k}(\BM y_T)[\BM Y]}{P_{\HAT{\BM\beta}}[\BM Y]}\Big)\Big]
\\={}&\Exp_{\BM Y\sim P_{\HAT{\BM\beta}}}\Big[\lim_{k\to+\infty}\psi\Big(\frac{Q_{\BM\theta_k}(\BM y_T)[\BM Y]}{P_{\HAT{\BM\beta}}[\BM Y]}\Big)\Big]
\\={}&\Exp_{\BM Y\sim P_{\HAT{\BM\beta}}}\Big[\psi\Big(\frac{\lim_{k\to+\infty}Q_{\BM\theta_k}(\BM y_T)[\BM Y]}{P_{\HAT{\BM\beta}}[\BM Y]}\Big)\Big]
\\={}&\Exp_{\BM Y\sim P_{\HAT{\BM\beta}}}\Big[\psi\Big(\frac{P_{\HAT{\BM\beta}}[\BM Y\mid\BM y_T]}{P_{\HAT{\BM\beta}}[\BM Y]}\Big)\Big]
\\={}&\Exp_{\BM Y\sim P_{\HAT{\BM\beta}}}\Big[\psi\Big(\frac1{P_{\HAT{\BM\beta}}[\BM y_T]}\Big)\Big]
,}
which implies
\AL{\lim_{k\to+\infty}\Exp_{\BM Y\sim P_{\HAT{\BM\beta}}}\Big[\psi\Big(\frac{Q_{\BM\theta_k}(\BM y_T)[\BM Y]}{P_{\HAT{\BM\beta}}[\BM Y]}\Big)\Big]=\min_{\BM\theta}\Exp_{\BM Y\sim P_{\HAT{\BM\beta}}}\Big[\psi\Big(\frac{Q_{\BM\theta}(\BM y_T)[\BM Y]}{P_{\HAT{\BM\beta}}[\BM Y]}\Big)\Big].}
This suggests that the sequence $\{\BM\theta_k\}_{k\ge1}$ is asymptotically optimal for the objective Eq.~\eqref{eq:q-equiv-obj}. Meanwhile, for any other parameter sequence $\{\TLD{\BM\theta}_k\}$ where $Q_{\TLD{\BM\theta}_k}(\BM y_T)$ converge to a distribution other than $P_{\HAT{\BM\beta}}\mathbin|\BM y_T$ for some $\BM y_T$ with $\#\OP{supp}(P_{\HAT{\BM\beta}}\mathbin|\BM y_T)>1$, then $\frac{\lim_{k\to+\infty}Q_{\TLD{\BM\theta}_k}(\BM y_T)[\BM Y]}{P_{\HAT{\BM\beta}}[\BM Y]}\mathbin{\Big|}\BM y_T$ is non-degenerate. By Fatou's lemma and Jensen's inequality with strict convexity,
\AL{
&\lim_{k\to+\infty}\Exp_{\BM Y\sim P_{\HAT{\BM\beta}}}\Big[\psi\Big(\frac{Q_{\BM\theta_k}(\BM y_T)[\BM Y]}{P_{\HAT{\BM\beta}}[\BM Y]}\Big)\Big]
\\\ge{}&\Exp_{\BM Y\sim P_{\HAT{\BM\beta}}}\Big[\lim_{k\to+\infty}\psi\Big(\frac{Q_{\BM\theta_k}(\BM y_T)[\BM Y]}{P_{\HAT{\BM\beta}}[\BM Y]}\Big)\Big]
\\={}&\Exp_{\BM Y\sim P_{\HAT{\BM\beta}}}\Big[\psi\Big(\frac{\lim_{k\to+\infty}Q_{\BM\theta_k}(\BM y_T)[\BM Y]}{P_{\HAT{\BM\beta}}[\BM Y]}\Big)\Big]
\\>{}&\psi\Big(\Exp_{\BM Y\sim P_{\HAT{\BM\beta}}}\Big[\frac{\lim_{k\to+\infty}Q_{\BM\theta_k}(\BM y_T)[\BM Y]}{P_{\HAT{\BM\beta}}[\BM Y]}\Big]\Big)
\\={}&\psi\Big(\Exp_{\BM Y\sim P_{\HAT{\BM\beta}}}\Big[\frac1{P_{\HAT{\BM\beta}}[\BM y_T]}\Big]\Big)
\\={}&\min_{\BM\theta}\Exp_{\BM Y\sim P_{\HAT{\BM\beta}}}\Big[\psi\Big(\frac{Q_{\BM\theta}(\BM y_T)[\BM Y]}{P_{\HAT{\BM\beta}}[\BM Y]}\Big)\Big]
.}
This suggests that $\{\TLD{\BM\theta}_k\}_{k\ge1}$ is not asymptotically optimal. Note that those $\BM y_T$ with $\#\OP{supp}(P_{\HAT{\BM\beta}}\mathbin|\BM y_T)=1$ has no influence, because in that case $Q_{\BM\theta}(\BM y_T)$ is degenerate and thus does not depend on $\BM\theta$. Hence, any asymptotically optimal parameter sequence $\{\BM\theta_k\}_{k\ge1}$ for the objective Eq.~\eqref{eq:q-equiv-obj} must also converge to $P_{\HAT{\BM\beta}}\mathbin|\BM y_T$. Therefore, the objectives Eq.~\eqref{eq:q-orig-obj} and Eq.~\eqref{eq:q-equiv-obj} are equivalent.
\end{proof}

\ASSEC{Proof of \PRPref{cplx}}{app:pf-6}
\begin{proof}[Proof of \textnormal{(i)}]
Since there are $T+1$ times and $n$ nodes, then there are $\OP O(Tn)$ pseudolikelihoods to be computed in total. The time complexity to compute the pseudolikelihood of a node at a time is at most proportional to the number of edges connecting to that node, and there are $m$ edges in total, so the total time complexity to compute all pseudolikelihoods is $\OP O(T(n+m))$. Furthermore, since the backpropagation algorithm has the same complexity as the forward computation, 
the overall time complexity of an iteration is still $\OP O(T(n+m))$.
\end{proof}

\begin{proof}[Proof of \textnormal{(ii)}]
To predict the probabilities $q_{t,u}^\SIRI$ and $q_{t,u}^\SIRR$, the time complexity is $\OP O(T(n+m))$ due to the graph neural network $Q_{\BM\theta}$. To generate a snapshot, the two main steps are sorting probabilities and maintaining counters. Sorting the $n$ probabilities $q_{t,u}^\SIRI$ of all nodes $u\in\CAL V$ takes $\OP O(n\log n)$ time at each $t$. Aggregating and updating the counters $\rho_{t,u}$ for all nodes $u\in\CAL V$ at each $t$ take $\OP O(n+m)$ time in total, because each edge is involved in $\OP O(1)$ operations. Since sampling a history needs to generate $T$ snapshots, then the total time complexity is $\OP O(T(n\log n+m))$.
\end{proof}

\ASEC{The Proposal in M--H MCMC}{app:gnn}

In this section, we detail our design of the proposal $Q_{\BM\theta}$ in M--H MCMC.

\ASSEC{Neural Architecture}{app:gnn-arch}
The backbone of $Q_{\BM\theta}$ is an Anisotropic GNN with edge gating mechanism \cite{bresson2018an,joshi2020learning,dimes}. Let $\BM g_u^{\ell,1}$ and $\BM g_{u,v}^{\ell,2}$ denote the node and edge embeddings at layer $\ell$ associated with node $u$ and edge $(u,v)$, respectively. The embeddings at the next layer is propagated with an anisotropic message passing scheme:
\AL{
\BM g_u^{\ell+1,1} &:= \BM g_u^{\ell,1}\! + a(\OP{BN}(\BM W^{\ell,1} \BM g_u^{\ell,1}\!+\!\!\!\underset{v\in\mathcal{N}_u}{\mathfrak A}\!\!(\sigma(\BM g_{u,v}^{\ell,2})\odot(\BM W^{\ell,2}\BM g^{\ell,1}_v)))),\\
\BM g_{u,v}^{\ell+1,2} &:= \BM g_{u,v}^{\ell,2} + a(\OP{BN}(\BM W^{\ell,3} \BM g^{\ell,2}_{u,v} + \BM W^{\ell,4} \BM g^{\ell,1}_u\! + \BM W^{\ell,5} \BM g^{\ell,1}_v)).
}
where $\BM W^{\ell,1},\dots,\BM W^{\ell,5}\in\BB R^{d\times d}$ are learnable parameters of layer $\ell$, $a$ denotes the activation function (we use $\OP{SiLU}$ \cite{silu} in this paper), $\OP{BN}$ denotes the Batch Normalization operator \cite{batch-norm}, $\mathfrak A$ denotes the aggregation function (we use mean pooling in this work), $\sigma$ denotes the sigmoid function, and $\odot$ denotes the Hadamard product. A Multi-Layer Perceptron (MLP) is appended after the GNN to produce the final outputs. The node inputs $\BM g_u^{0,1}$ are initialized by feeding $y_{T,u}$ into a linear layer. The edge inputs $\BM g^{0,2}_{u,v}$ are learnable parameters.

\ASSEC{Sampling Scheme}{app:gnn-samp}
To satisfy the condition in \THMref{q-equiv-obj} and not to generate waste samples, we require 
that $\OP{supp}(Q_{\BM\theta}(\BM y_T))=\OP{supp}(P\mathbin{|}\BM y_T)$ for any snapshot $\BM y_T$. A sufficient condition for this requirement is that $Q_{\BM\theta}(\BM y_T)[\BM y_{t}\mathbin|\BM y_{t+1}]>0$ iff $P_{\HAT{\BM\beta}}[\BM y_{t+1}\mathbin|\BM y_{t}]>0$ for every $t=1,\dots,T$, as long as $P[\BM y_0]$ is nowhere vanishing. Hence, we design a sampling scheme according to this sufficient condition. We sample a history in the reverse temporal order. Provided that $\BM y_{t+1}$ is already generated, we next describe how to generate $\BM y_t$.

Given the observed snapshot $\BM y_T$, the GNN $Q_{\BM\theta}(\BM y_T)$ predicts two probabilities $0<q^\SIRI_{t,u},\,q^\SIRR_{t,u}<1$ for each node $u\in\CAL V$ at time $t$. There are two stages at time $t$. In the first stage, for each node $u\in\CAL V$, if $y_{t+1,u}=\SIRR$, we randomly change its state to $\SIRI$ with probability $q^\SIRR_{t,u}$. Let $y'_{t,u}$ denote the state of node $u$ after the first stage. In the second stage, we sort the nodes $\CAL V$ into $v_1,\dots,v_n$ so that $q^\SIRI_{t,v_1}\ge\dots\ge q^\SIRI_{t,v_n}$. We also maintain a counter $\rho_{t,u}$ for each node $u$ to indicate the remaining chance to be infected from a neighbor. The counters $\rho_{t,u}$ are initialized as one plus the degree of the node $u$. Then, we decide whether to change the state of $v_i$ to $\SIRS$ sequentially in the order $v_1,\dots,v_n$. If $y'_{t,v_i}\ne\SIRI$, then we set $y_{t,v_i}:=y'_{t,v_i}$. If $y'_{t,v_i}=\SIRI$ and $\min_{u\in\{v_i\}\cup\CAL N_{v_i}}\rho_{t,u}\le 1$, then we set $y_{t,v_i}=\SIRI$. Otherwise, we set $y_{t,v_i}:=\SIRS$ with probability $q^\SIRI_{t,v_i}$ or $y_{t,v_i}:=\SIRI$ with probability $1-q^\SIRI_{t,v_i}$. If $y'_{t,v_i}=\SIRI$ and $y_{t,v_i}=\SIRS$, then we decrease $\rho_{t,u}$ by $1$ for all $u\in\{v_i\}\cup\CAL N_{v_i}$. If $y'_{t,v_i}=\SIRI$ and $y_{t,v_i}=\SIRI$, then we set $\rho_{t,u}:=+\infty$ for all $u\in\{v_i\}\cup\CAL N_{v_i}$. After the second stage, we get the states $\BM y_t$ at time $t$, and then we use it to generate $\BM y_{t-1}$, and so on.

\ASEC{Detailed Experimental Setting}{app:setting}
In this section, we give detailed description of datasets, baselines, and reproducibility.
\ASSEC{Datasets}{app:exp-data}
We use 3 types of datasets. 
\textbf{(D1) Synthetic graphs and diffusion.} We generate random undirected graphs with 
1,000 nodes using the Barab\' asi--Albert (BA) model \cite{barabasi-albert} with attachment 4 and the Erd\H os--R\' enyi (ER) model \cite{erdos-renyi} with edge probability $0.008$. We simulate SI and SIR diffusion for $T=10$ with the infection rate $0.1$ and the recovery rate $0.1$. We randomly select $5\%$ nodes as diffusion sources.
\textbf{(D2) Synthetic diffusion on real graphs.} We use two real graphs and simulate diffusion on them. Oregon2\footnote{http://snap.stanford.edu/data/Oregon-2.html} \cite{oregon2} is a collection of graphs representing the peering information in autonomous systems, and we use the graph of May 26, 2001, the largest one. Prost 
\cite{data-prost} is a bipartite graph representing prostitution reviews in an online forum. We simulate SI and SIR diffusion for $T=15$ with the infection rate $0.1$ and the recovery rate $0.05$. We randomly select $10\%$ nodes as diffusion sources.
\textbf{(D3) Real diffusion on real graphs.} To evaluate how \method{} generalizes to real-world diffusion, we use 4 real-world diffusion datasets. 
BrFarmers\footnote{https://usccana.github.io/netdiffuseR/reference/brfarmers.html} \cite{data-farmers-orig,data-farmers} incorporates 
diffusion of technological innovations among Brazilian farmers. It is an SI-like diffusion, where an infection means that a farmer hears about the new technology from a friend and adopts it. 
Pol\footnote{https://networkrepository.com/rt-pol.php} \cite{data-pol,netrepo} is a temporal retweet network about a U.S.\ political event. It is an SI-like diffusion, because when a user retweets or is retweeted, they must have known about the event. 
Covid\footnote{https://data.cdc.gov/Public-Health-Surveillance/United-States-COVID-19-Community-Levels-by-County/3nnm-4jni} is a dataset of Covid Community Levels 
from Feb 23, 2022 to Dec 21, 2022 released by CDC, where nodes are counties. We build a graph by connecting each node with its $10$ nearest neighbors according to latitudes and longitudes. It is an SIR-like diffusion as follows. When a county becomes medium or high level for the first time, the county is ``infected.''
After the last time a county becomes low level and does not change again, the county is ``recovered.''
Hebrew 
\cite{data-heb} is a temporal retweet network among Hebrew tweets about an Israeli election event. It is an SIR-like diffusion as follows. For users who retweet at most once and are never retweeted, they never become influential in this event, so they are ``susceptible.'' For users who retweet at least twice or are retweeted by others, they actively involve or influence other users in the event, so they are ``infected.'' For ``infected'' users, after the last time that they are retweeted, they are no longer influential, so they become ``recovered.'' 

\ASSEC{Baselines}{app:exp-bl}
We compare \method{} with 2 types of baselines. 
\textbf{(B1) Supervised methods for time series imputation.} Diffusion history reconstruction can be alternatively formulated as time series imputation on graphs. Therefore, we also compare \method{} with state-of-the-art time series imputation methods, including BRITS \cite{brits} for multivariate time series, and GRIN \cite{grin} and SPIN \cite{spin} for graph time series. Since these methods are all based on supervised learning, we use our estimated diffusion parameters to simulate diffusion histories as training data. We follow the hyperparameters of baselines, except adjusting the batch size to fit in memory.
\textbf{(B2) MLE-based methods for diffusion history reconstruction.} 
To date, few works have studied diffusion history reconstruction, and all of them are based on the MLE formulation. We compare \method{} with state-of-the-art methods DHREC \cite{pcvc} and CRI \cite{cri}. DHREC reduces the MLE formulation to the Prize Collecting Dominating Set Vertex Cover (PCDSVC) problem and uses a greedy algorithm to solve PCDSVC. It requires the knowledge of the diffusion model parameter. Therefore, we feed our estimated diffusion parameters to it. CRI designs a heuristic method based on clustering and reverse infection. It can estimate infection times but cannot estimate recovery times.

\ASSEC{Reproducibility \& Implementation Details}{app:repro}

Experiments were run on Intel Xeon CPU @ 2.20GHz and NVIDIA Tesla P100 16GB GPU. Our source code is publicly available at \url{https://github.com/q-rz/KDD23-DITTO}. All datasets are publicly available. For each dataset, we will either provide a link to it or include it in our code repository.

For \method{}, we optimize $\HAT{\BM\beta}$ for $I=500$ iterations. The proposal $Q_{\BM\theta}$ is a 3-layer GNN followed by a 2-layer MLP with hidden size 16. We train $Q_{\BM\theta}$ for $J=500$ steps for D1 and D2, $J=2,000$ for BrFarmers, $J=300$ for Pol, $J=250$ for Covid, and $J=200$ for Hebrew. We use batch size $K=10$ for D1, BrFarmers, and Covid, and $K=2$ for D2, Pol, and Hebrew. We use the AdamW \cite{adamw} optimizer with learning rate 0.003 for $\HAT{\BM\beta}$ and 0.001 for $Q_{\BM\theta}$. After training, we run M--H MCMC for $S=10$ iterations with $L=100$ samples per iteration and $\eta=0.5$ for moving average. For the initial distribution $P[\BM y_0]$, we use the coefficient $\gamma=1$ in MCMC. For supervised imputation methods in B1, we use the estimated diffusion parameters (unless specified) to generate training data. For GRIN and SPIN, we train them for 1,000 steps with batch size 1. We follow the other hyperparameters of these methods. For MLE-based methods in B2, we feed estimated diffusion parameters to them.

\ASEC{Additional Experiments}{app:exp}
In this section, we present addtional experimental results to answer \RQref{timespan} and \RQref{abla}.

\begin{figure}[t]
\begin{center}
\begin{subfigure}[t]{0.48\columnwidth}\centering
\includegraphics[width=\textwidth]{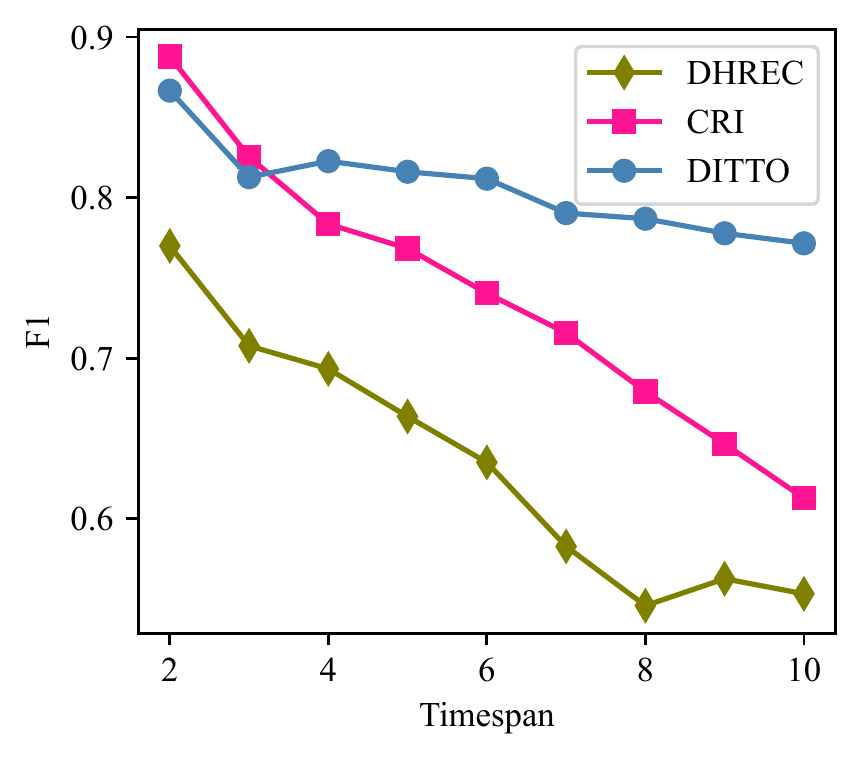}
\caption{F1 vs timespan $T$.}
\label{fig:timespan-f1}
\end{subfigure}
\hfill
\begin{subfigure}[t]{0.48\columnwidth}\centering
\includegraphics[width=\textwidth]{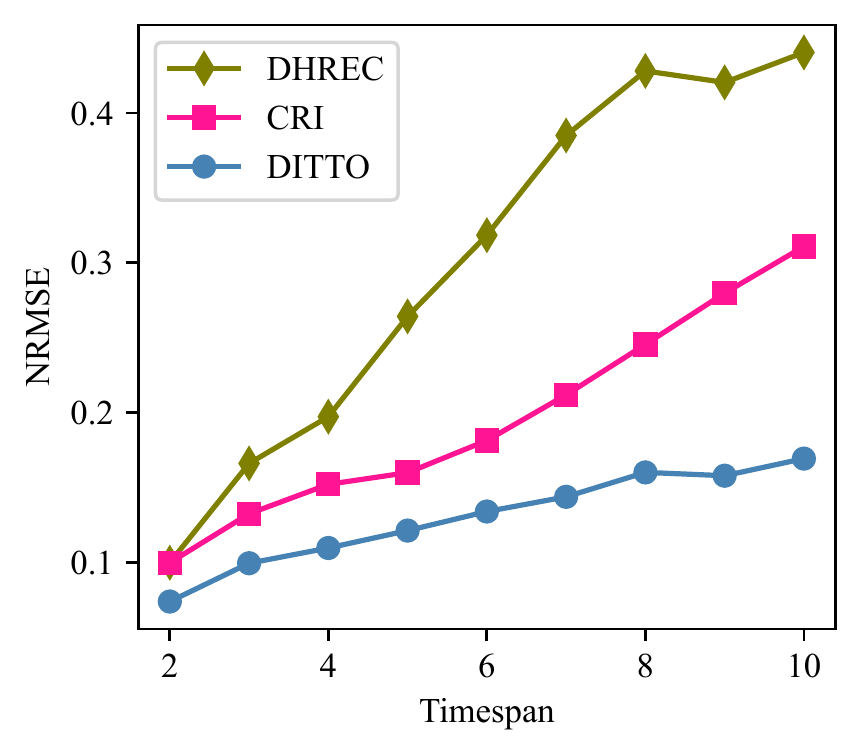}
\caption{NRMSE vs timespan $T$.}
\label{fig:timespan-nrmse}
\end{subfigure}
\end{center}
\caption{Performance vs timespan $T$.}
\label{fig:timespan}
\end{figure}

\begin{figure}[t]
\begin{center}
\begin{subfigure}[t]{0.48\columnwidth}\centering
\includegraphics[width=\textwidth]{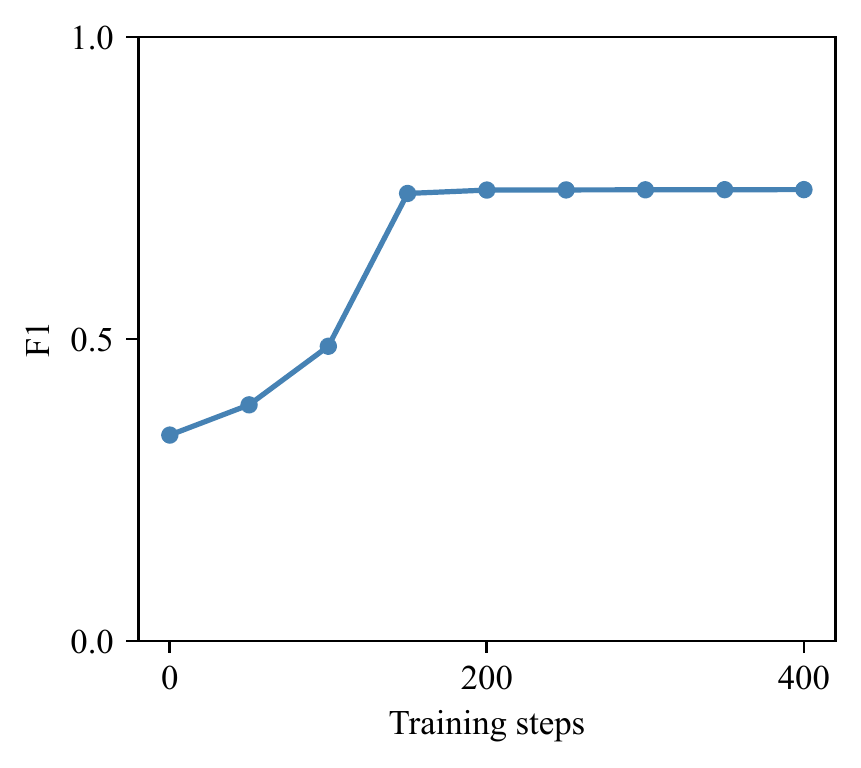}
\caption{F1 vs training steps.}
\label{fig:qstep-f1}
\end{subfigure}
\hfill
\begin{subfigure}[t]{0.48\columnwidth}\centering
\includegraphics[width=\textwidth]{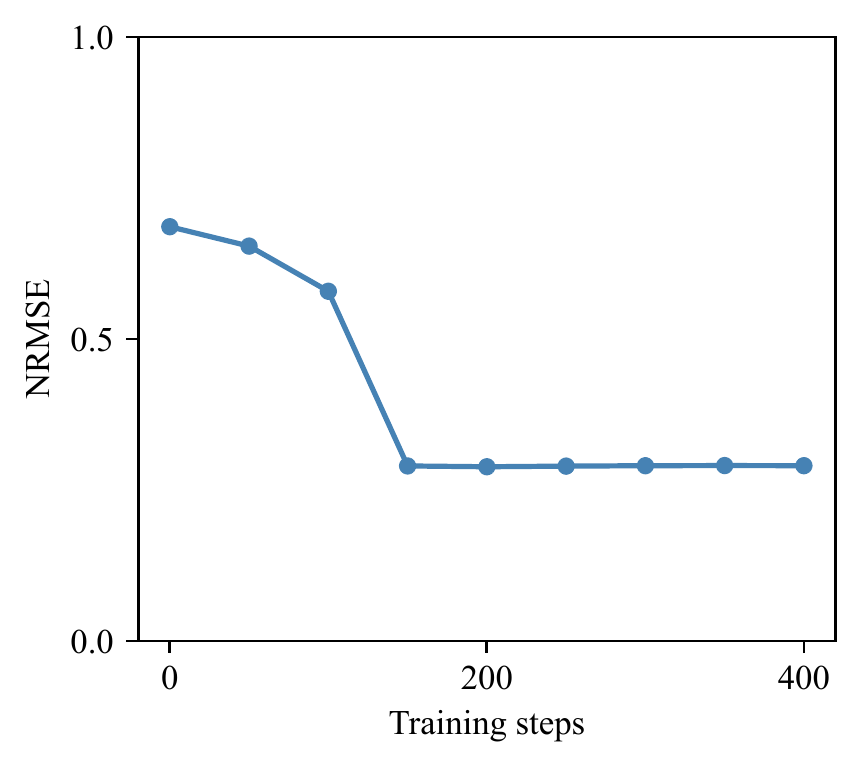}
\caption{NRMSE vs training steps.}
\label{fig:qstep-nrmse}
\end{subfigure}
\end{center}
\caption{Performance vs training steps.}
\label{fig:qstep}
\end{figure}

\ASSEC{Effect of Timespan}{app:exp-timespan}
As the timespan $T$ increases, the search space of possible histories grows exponentially and thus the uncertainty of the history grows accordingly. Hence, it is helpful to investigate the effect of timespan $T$. To answer \RQref{timespan}, we vary the timespan $T$ from 2 to 10 and compare the performance of \method{} and MLE-based methods. The results are shown in Fig.~\ref{fig:timespan}. As is expected, the performance of all methods degrades as $T$ increases. Nonetheless, the performance of \method{} degrades slower than that of MLE-based methods. The results demonstrate that \method{} can better handle the uncertainty induced by the increase in the timespan $T$.

\ASSEC{Ablation Study}{app:exp-abla}
To answer \RQref{abla}, we conduct ablation study on the effect of the number of training steps. We vary the number of training steps from 0 to 400 for the Pol dataset and compare the performance in terms of F1 and NRMSE. The results are shown in Fig.~\ref{fig:qstep}. When the number of training steps is less than 200, as the number of training steps increases, the performance of \method{} improves accordingly. When the number of training steps is more than 200, the performance does not change because the proposal already converges. The results suggest that the learned proposal in \method{} is indeed beneficial for M--H MCMC.

\ASEC{Limitations \& Future Work}{app:concl}

One limitation of \method{} is that the history $\HAT{\BM Y}$ reconstructed by Eq.~\eqref{eq:hit-to-hist} is not necessarily feasible under the SIR model. However, the perfect feasibility under the SIR model has limited significance in practice, because the SIR model is often considered as an over-simplification of real-world diffusion. Meanwhile, an alternative solution is to use the samples generated by M--H MCMC, as they are guaranteed to be feasible. 
Another potential limitation is that our theoretical analyses are based on small diffusion parameters, which is indeed the case for most real-world data \cite{o1987epidemiology,goh2009inflammatory,vanderpump2011epidemiology}. Meanwhile, there might exist situations where infection rates are large. The analyses under large diffusion parameters is beyond the scope of our work. To our best knowledge, no literature has studied diffusion history reconstruction with large infection rates, which leads to an interesting future research direction that is worth an independent investigation. Besides that, there are a number of other directions that are worth future study, including extending to 
diffusion models other than the SI/SIR model, incorporating node and edge attributes to allow heterogeneity, and improving the expressiveness of the proposal to further accelerate the convergence of M--H MCMC.




\end{document}